\documentclass[letterpaper]{article}

\usepackage{necessary_files/arxiv}

\usepackage[utf8]{inputenc} 
\usepackage[T1]{fontenc}   
\usepackage{hyperref}       
\usepackage{url}            
\usepackage{booktabs}       
\usepackage{amsfonts}       
\usepackage{microtype}      

\usepackage{graphicx}
\usepackage{amssymb}
\usepackage{amsmath}
\usepackage{enumitem}
\usepackage{amsfonts}

\usepackage{amsthm}
\theoremstyle{definition}

\newtheorem{definition}{Definition}[section]
\newtheorem*{definition*}{Definition*}
\newtheorem{lemma}{Lemma}

\newtheorem*{remark*}{Remark}

\usepackage{thm-restate}
\usepackage{booktabs}
\usepackage{algorithm}
\usepackage{algorithmic}

\usepackage[dvipsnames]{xcolor}

\usepackage{natbib}

\usepackage[english]{babel}

\usepackage{times}
\usepackage{helvet} \usepackage{courier}
\newcommand{\AutoAdjust}[3]{\mathchoice{ \left #1 #2  \right #3}{#1 #2 #3}{#1 #2 #3}{#1 #2 #3} }
\newcommand{\Xcomment}[1]{{}}

\newcommand{\InBrackets}[1]{\AutoAdjust{[}{#1}{]}}

\newcommand{\Prx}[2][]{\operatorname{\mathbf{Pr}}_{#1}\InBrackets{#2}}

\newcommand{\E}{\mathbf{E}}

\DeclareMathOperator*{\argmax}{arg\,max}

\newcommand{\noaccents}[1]{#1}
\newcommand{\newagentvar}[3][\noaccents]{%
\expandafter\newcommand\expandafter{\csname #2\endcsname}{#1{#3}}%
\expandafter\newcommand\expandafter{\csname #2s\endcsname}{#1{\boldsymbol{#3}}}%
\expandafter\newcommand\expandafter{\csname #2smi\endcsname}[1][i]{#1{\boldsymbol{#3}}_{-##1}}%
\expandafter\newcommand\expandafter{\csname #2i\endcsname}[1][i]{#1{#3}_{##1}}%
\expandafter\newcommand\expandafter{\csname #2ith\endcsname}[1][i]{#1{#3}_{(##1)}}%
}

\newcommand{\newvecagentvar}[3][\noaccents]{%
\expandafter\newcommand\expandafter{\csname #2\endcsname}{#1{\boldsymbol{#3}}}%
\expandafter\newcommand\expandafter{\csname #2s\endcsname}{#1{\boldsymbol{#3}}}%
\expandafter\newcommand\expandafter{\csname #2smi\endcsname}[1][i]{#1{\boldsymbol{#3}}_{-##1}}%
\expandafter\newcommand\expandafter{\csname #2i\endcsname}[1][i]{#1{\boldsymbol{#3}}_{##1}}%
\expandafter\newcommand\expandafter{\csname #2ith\endcsname}[1][i]{#1{#3}_{(##1)}}%
}

\newagentvar{alloc}{x}
\newagentvar{pay}{p}
\newagentvar{val}{v}
\newagentvar{util}{u}
\newagentvar{payment}{p}
\newagentvar{dist}{F}

\newcommand{\Uniform}{\mathrm{Uniform}}

\definecolor{BLACK}{named}{black}

\newcommand{\OMLP}{{\color{black}\noindent occupancy-measure LP}}
\newcommand{\ContextualOccupancyIndex}[0]{{\color{black}\noindent Contextual Occupancy Index}}
\newcommand{\ContextualBudgetBandit}[0]{{\color{black}\noindent Contextual Budget Bandit}}
\newcommand{\CBB}[0]{{\color{black}\noindent \text{CBB}}}
\newcommand{\ContextSpecificBudgetConstraint}[0]{{\color{black}\noindent Contextual Budget Constraint}}
\newcommand{\policy}{{\color{black}\noindent\texttt{Policy}}}

\newcommand{\ContextualOccupancySoftBudgetPolicy}[0]{{\color{black}\noindent Contextual Occupancy Soft Budget Policy}}

\newcommand{\Reward}{{\text{$\mathcal R$eward}}}
\newcommand{\Fairness}{{\text{$\mathcal F\!$\small{airness}}}}

\newcommand{\FlexibleBudgetAllocCWIndexPolicy}{\text{Flexible-Budget-Allocation-Contextual-Occupancy-Index Policy}}

\newcommand{\ContextualOccupancyIndexPolicy}[0]{{\color{black}\noindent COcc}}
\newcommand{\BranchAndBound}[0]{{\color{black}\noindent Branch And Bound}}
\newcommand{\ZoomingBranch}[0]{{\color{black}\noindent Mitosis}}
\newcommand{\Mitosis}[0]{{\color{black}\noindent Mitosis}}
\newcommand{\Random}[0]{{\color{black}\noindent Random}}
\newcommand{\Greedy}[0]{{\color{black}\noindent Greedy}}
\newcommand{\VanillaWhittle}[0]{{\color{black}\noindent Vanilla Whittle}}

\newcommand{\TreeArm}[0]{{\color{black}\noindent \text{StemArm}}}
\newcommand{\Oracle}[0]{{\color{black}\noindent {\text{$\mathcal O$racle}}}}
\newcommand{\OracleSmall}[0]{{\color{black}\noindent {\text{$\text{$\mathcal O$racle}_\text{small}$}}}}
\newcommand{\LP}[0]{{\color{black}\noindent {\text{$\mathcal L$P}}}}

\newcommand{\frname}{412 Food Rescue}

\title{Contextual Budget Bandit for Food Rescue Volunteer Engagement}

\author{Ariana Tang\textsuperscript{1}, Naveen Raman\textsuperscript{2}, Fei Fang\textsuperscript{2}, Zheyuan Ryan Shi\textsuperscript{3}\\
\textsuperscript{1}University of Chicago, \textsuperscript{2}Carnegie Mellon University, \textsuperscript{3}University of Pittsburgh
}

\begin{document}

\maketitle

\begin{abstract}
Volunteer-based food rescue platforms tackle food waste by matching surplus food to communities in need.
These platforms face the dual problem of maintaining volunteer engagement and maximizing the food rescued. 
Existing algorithms to improve volunteer engagement exacerbate geographical disparities, leaving some communities systematically disadvantaged. 
We address this issue by proposing \ContextualBudgetBandit{}. \ContextualBudgetBandit{} incorporates context-dependent budget allocation in restless multi-armed bandits, a model of decision-making which allows for stateful arms. 
By doing so, we can allocate higher budgets to communities with lower match rates, thereby alleviating geographical disparities. 
To tackle this problem, we develop an empirically fast heuristic algorithm. 
Because the heuristic algorithm can achieve a poor approximation when active volunteers are scarce, we design the \ZoomingBranch{} algorithm, which is guaranteed to compute the optimal budget allocation.
Empirically, we demonstrate that our algorithms outperform baselines on both synthetic and real-world food rescue datasets, and show how our algorithm achieves geographical fairness in food rescue. 
\end{abstract}
\begin{figure}[t]
    \centering
    \includegraphics[width=0.7\linewidth]{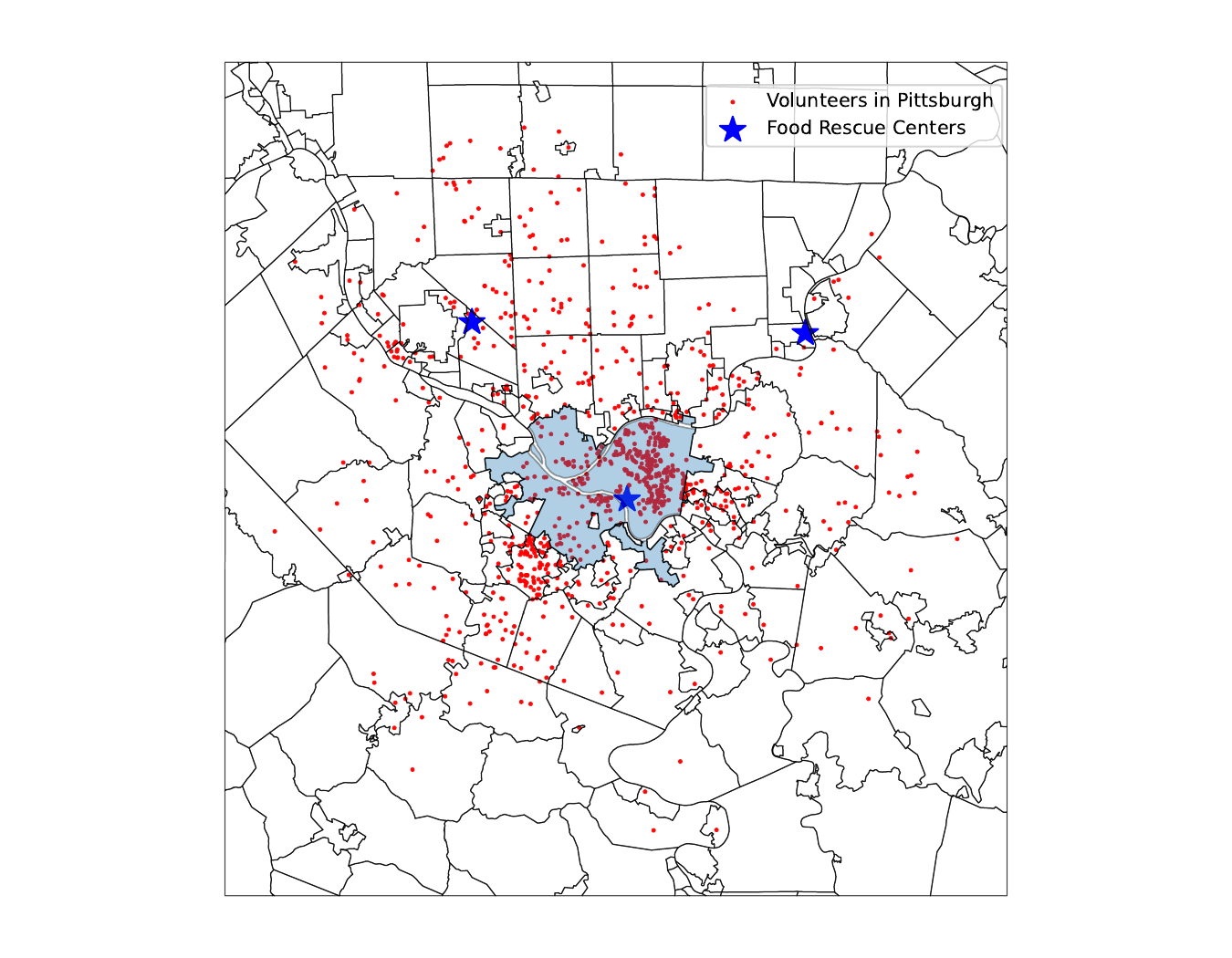}
    \caption{The picture shows volunteers and donation regions in real food rescue database. The central "popular" downtown region has 720 nearby volunteers clustered with an average distance of 7.7 km to their rescue spots. In contrast, the least favored northwestern region has only 85 volunteers on average and a longer average distance of 10.6 km.}
    \label{fig:ai4sg}
\end{figure}

\section{Introduction}

The world wastes up to 40\% of our food globally, translating to over 1.3 billion tons annually, while 1 in 7 people struggle to secure enough food every day~\citep{coleman2018household,conrad2018relationship}.
With their appearance in over 100 cities worldwide, food rescue platforms (FRP) receive safe, edible food donations from businesses like restaurants (``donors'') and distribute them to organizations serving low-resource communities (``recipients''). 
Our partner organization, \frname{},
is a large FRP with operations in over 25 different cities across the US. 
FRPs are able to scale due to volunteers, who transport food from donors to recipients. 
Essentially, volunteers claim ``rescues'' from an FRP's mobile app. 
After claiming the rescue, the app instructs them where to pick up and drop off the donation. 

The inclusion of volunteers in FRPs brings about inherent uncertainty due to changing volunteer behavior. 
Volunteer engagement is critical to FRP success, so FRPs have an urgent need to engage their volunteers while maximizing the amount of food rescued. 
A few studies have developed algorithms to improve volunteer engagement on FRPs by dynamically notifying volunteers about rescue trips~\citep{shi2021recommender,shi2024predicting,raman2024global}. However ~\citet{shi2021recommender} showed that such algorithms can backfire because they result in severe geographical disparity in food rescue outcomes (Figure~\ref{fig:ai4sg}).
In some regions such as downtown, the algorithm enjoyed almost 90\% completion rate, while in some outer suburbs, the completion rate dropped to 40\%. 

In our work, we study how to maintain volunteer engagement while combatting geographical disparities. 
The challenge is that volunteer behaviors evolve over time in response to notification patterns. 
To tackle this issue, we model food rescue volunteer engagement as a restless multi-armed bandit (RMAB) problem, a common model for online resource allocation~\citep{rmab_public_health,raman2024global}. 
We extend this model to incorporate geographical disparities with a \textit{context}, which corresponds to geographic information for each rescue trip. 
We then set notification quotas for different regions  so certain regions with scarce volunteers have higher budgets. 
Such an approach allows for flexibility in notifications without sacrificing overall performance. 

We make the following contributions: (1) We propose the \textit{\ContextualBudgetBandit{}} problem, which extends RMABs to situations with
context-dependent budget allocations. Such a problem is motivated by applications in food rescue, but can also model problems in a variety of domains such as digital agriculture and peer review (see Section~\ref{sec:other_areas}).
(2) We develop the \ContextualOccupancyIndexPolicy{}, a fast, empirically approximation algorithm which provides an upper bound to \ContextualBudgetBandit{}. We characterize cases where it fails with a constant factor; (3) We design the \ZoomingBranch{} algorithm which is guaranteed to compute the optimal budget allocation; and (4) We empirically demonstrate that our algorithms improve upon classical algorithm baselines.
 
\section{Related Works}
 
\paragraph{\ContextualBudgetBandit{}}

Contextual information is often present in bandit and can substantially improve decision quality, and has been extensively studied in classical Multi-Armed Bandits~\citep{bouneffouf2020survey-contextualBandit, langford2007-ContextualMAB}. Incorporating context with RMABs is a promising direction, as it enables modeling both individual arm's state dynamics and global contextual influences.
To our knowledge, four main works combine context into RMABs. \citet{Liang_Xu_Taneja_Tambe_Janson_2025-BayesianContextualRMAB}, \citet{mimouni25-DeepQLearningContextualRMAB}, and \citet{context_win} treat context as static side information that affects arms’ transitions and rewards. These studies focus on learning transition parameters online~\citep{Liang_Xu_Taneja_Tambe_Janson_2025-BayesianContextualRMAB} or approximating Whittle indices using deep learning~\citep{mimouni25-DeepQLearningContextualRMAB, context_win}.

The most relevant prior work is \citet{contextual_demand_rmab}, which models context as a global Markov-evolving state and develops a dual-decomposition-based index policy, which is a generalization of Whittle Index Policy. Their approach proves optimal under the assumption that contexts evolve deterministically and periodically. In contrast, we assume context evolves randomly under a Bayesian prior, which generalizes and weakens this assumption. We theoretically demonstrate that under this setting, Whittle index policy loses asymptotic optimality, underscoring the need for careful budget design in contextual RMABs.

All prior contextual RMAB work assumes a fixed activation budget per round, independent of context. In this paper, we allow the budget to vary by context, and show---both theoretically and empirically---that context-aware budget allocation improves performance. We design algorithm for calculating optimal budget allocation.

\paragraph{Fairness} is an increasingly important consideration in both business and non-profit organizations~\cite{Bertsima2012MS_EfficiencyFairnessTradeoff,NikhilGarg_2024_ServiceLevelAgreementFairness}. For RMAB, fairness is typically imposed on individual arms: \cite{Wang_Xiong_Li2024longTermFairnessOnArms} define fairness as requiring a minimum long-term activation fraction for each arm; \cite{Pradeep2022_SoftFairnessRMAB} propose a \textit{soft fairness constraint}, or by setting an upperbound on the number of decision epochs since an arm was last activated~(\cite{Pradeep2022_RMABfairness}). Fairness can also be defined over groups of arms. \cite{killian2023groupFairness} study minimax and max-Nash welfare objectives by imposing fairness on groups of arms, and ~\cite{Verma2024_GroupFairnessRMAB} enforce fairness with respect to the reward outcomes across groups. To the best of our knowledge, although RMABs with \textbf{contextual information} have been previously studied, our work is the first to consider fairness with respect to \textbf{context}.

\subsection{Generalization}\label{sec:other_areas}

While we ground our work in food rescue volunteer engagement, our model and algorithms are applicable to
many domains.

\paragraph{Digital agriculture} 
Agriculture chatbots empower smallholder farmers~\citep{guérin2024reportnsfworkshopsustainable}.
In collaboration with Organization X, we have a chatbot which sends nudges about farming practices to over 50,000 farmers in India, Kenya, and Nigeria.
However, nudges of different topics have different conversion rates. Pest control tips during the pest season address an urgent problem, resulting in high conversion rates. Meanwhile, watering tips are preventive measures, which often have lower conversion rates by the farmers. Thus, one would assign different nudging budgets to different topics of nudges, and model it as a \CBB{}. Each farmer is an arm. At each time step, we have a nudge topic as context, and we decide on a budget of how many farmers to notify and the arm selection of who to notify. Rewards are determined based on farmer's engagement response. 

\paragraph{Peer review}
Journals select reviewers where selection impacts future reviewer availability~\citep{payan2021will}. For a given paper, the goal is to select a subset of reviewers with the relevant expertise. However, submissions differ from one another. For example, submissions that are extra long, that involves heavy theoretical analysis, or that do not study the trendy topics might have lower chance of getting reviewers. Thus, when the editor plans reviewing invitations over time, they would want to send different numbers of invitations to different kinds of submissions, and model it as a \CBB{}. Each potential reviewer is an arm. At each time step, we have a submission type as context, and we decide on a budget of how many potential reviewers to reach out to, and the arm selection of who to reach out to. Rewards are determined based on the reviewers' response.

\paragraph{Email campaigns}

Email marketing platforms often face the challenge of optimizing user engagement while avoiding excessive spam. We model this setting as a \CBB{}, where each user corresponds to an arm, and at each time step, the platform observes contextual information such as campaign features (e.g., seasonal promotions) and user attributes (e.g., location, browsing history, and engagement profile)~\citet{mimouni25-DeepQLearningContextualRMAB}. The platform decides both (i) the budget of how many users to target and (ii) which users to send emails to, balancing personalization with resource constraints.

The reward is based on user engagement outcomes, such as opening an email, clicking on a link, or making a purchase. Since user engagement evolves over time—even when no emails are sent—this setting naturally requires a contextual restless bandit formulation. By leveraging contextual information, the system can dynamically adapt to changing user preferences and campaign effects, allocating larger budgets to contexts with higher expected conversion rates while maintaining long-term engagement.
\section{Preliminary Background}

A Restless Multi-Armed Bandit (RMAB) is defined by:
\begin{equation*}
    \langle N, \mathcal{S}, \mathcal{A},\lbrace r_{i} \rbrace_{i \in [N]}, \lbrace P_{i} \rbrace_{i \in [N]}\rangle.
\end{equation*}
Each arm $i\in [N]:=\{1, 2, \ldots, N\}$ is an independent Markov Decision Process, with state space $\mathcal S_i=\{0, 1\}$ and binary action space $ \mathcal A_i = \{0, 1\}$. Action $0$ corresponds to idling the arm while action $1$ corresponds to pulling the arm. 
The reward function for each arm $r_i:\mathcal S_i\times \mathcal A_i \to \mathbb R$ maps state-action pairs to a reward. $P_i$ is the transition kernel for each arm $i$. The overall system state at time $t$ is $\mathbf{s}^t=(s_1^t, s_2^t,\ldots, s_N^t)$, and the decision maker selects action $\mathbf{a}^t=(a_1^t, a_2^t,\ldots, a_N^t)$ subject to a budget constraint:
$$
\sum_{i\in [N]} a_i^t \le B,\quad \forall\, t=1,2,\ldots,
$$
which limits the number of arms that can be pulled in every time step. The objective is to design a policy that maps the current state $\mathbf{s}^t$ to an action vector $\mathbf{a}^t$ that maximizes the average reward over all arms and over an infinite time horizon.

Solving optimal policy for RMAB is PSPACE-hard~(\citet{RMAB_PSPACE_HARD_PAPA1999}). A widely-adopted approach for tackling the computational complexity inherent in RMABs is the Whittle Index Policy. The \emph{Whittle Index} for each arm~$i$'s state $s_i\in \mathcal S_i$ is $w_{i}(s_{i}) = \min\limits_{w} \{w | Q_{i,w}(s_{i},0) = Q_{i,w}(s_{i},1)\}$, defined using the standard Bellman $Q$-function and $V$ (value) function:
\begin{align*}
    & Q_{i,w}(s_{i},a_{i}) = -w a_{i} + r_{i}(s_{i},a_{i}) + \gamma \sum_{s'} P_{i}[s_{i},a_{i},s'] V_{i,w}(s')\\
    & V_{i,w}(s') = \max\limits_{a} Q_{i,w}(s',a)
\end{align*}
Under the crucial condition of \textbf{indexability}---which requires that the set of states where it is optimal to activate an arm decreases monotonically as $w$ (which can be think of as a penalty for pulling an arm) increases---the Whittle index is well-defined and interpretable as the marginal value of activating an arm.
At each time step, the \emph{Whittle Index Policy} pulls the~$B$ arms with the highest Whittle Indices. In this way, it decouples the multi-armed problem into a collection of single-arm problems. 
The Whittle index policy is asymptotically optimal under regularity conditions as the number of arms goes to infinity ~\citep{Gittins1991RMABBookChapter6}.

\section{\ContextualBudgetBandit{}}
Because traditional methods to maintain volunteer engagement can lead to geographical disparity~\citep{shi2021recommender}, we pursue an intuitive solution where we allocate different notification budget to different regions. 
To do this, we need to augment the standard RMAB model with variability in transition and reward across time, and the flexibility to adjust budget accordingly.
In this section, we will introduce the \ContextualBudgetBandit{} model and multiple algorithms for it.
\subsection{The \ContextualBudgetBandit{} Model}

A \ContextualBudgetBandit{} (\CBB{}) is defined by the tuple
$$\langle N, \mathcal{S}, \mathcal{A}, K, \lbrace r_{i}^{k} \rbrace_{i \in [N], k \in [K]}, \lbrace P_{i}^{k} \rbrace_{i \in [N], k \in [K]}, \mathcal{F}\rangle.$$
Departing from the standard RMAB model, we introduce the $[K] = \{1, 2, \ldots, K\}$ (finite) \textbf{contexts}. A Borel measure $\mathcal{F}$ on $[K]$ specifies the distribution over these contexts, which is known by the decision maker. At each time step, a new context is sampled with respect to $\mathcal F$ and globally applies to all arms. $r_i^k$, $P_i^k, \forall k\in [K]$ are the reward function and the transition probability kernels \emph{specific to context~$k$}. We write reward function $r_i^k$ in the expanded form as $r_i(s, a; k)$ (and $P[s\to s'|a, k]$ for $P_i^k$ respectively). Context and state transitions are \emph{independent}.

A policy $\pi$ for \CBB{} (i) pre-specifies budget allocation $\vec B := (B_k)_{k\in [K]}$ and (ii) maps $(\mathbf s^t, k^t)$ to $\mathbf a^t$. The design objective is to maximize the average expected reward:
\begin{equation*}
    \Reward(\pi):= \lim_{T \to \infty} \frac{1}{T} \sum_{t=1}^T \mathbb{E}_{(\mathbf a, \mathbf s) \sim \pi, k^t\sim \mathcal F} \Bigl[  \sum_{i \in [N]} r_i(s_i^t, a_i^t; k^t)\Bigr].
\end{equation*}

In standard RMAB, there is a budget for the number of arms pulled at each time point. We generalize this budget notion for \CBB{} as the following \ContextSpecificBudgetConstraint{}:
\begin{definition}[\ContextSpecificBudgetConstraint{}]
    \label{def:context_specific_budget_constraint}
    A policy is said to satisfy \emph{\ContextSpecificBudgetConstraint{}} if the number of arms pulled at each time step is constrained by a fix quantity $B_k$ contingent on context (Constraint I), while the expected budget usage is still bounded by $B$ (Constraint II):
    \begin{align*}
        &\sum_{i\in[N]} a_i^t \mathbb{I}(k^t=k) \le B_k, \forall t, k\tag{Constraint I}\\
        &\mathbb E_{k\sim \mathcal F} [B_k] = \sum_{k\in [K]} f_k B_k \le B \tag{Constraint II}.
    \end{align*}
\end{definition}

The vanilla Whittle Index Policy for standard RMAB (henceforth \VanillaWhittle{}) can be directly applied to \CBB{}: it uses a uniform budget for each context $\vec B = (B, \ldots, B)$. It satisfies the \ContextSpecificBudgetConstraint{}. However, the following theorem shows that its performance can be arbitrarily bad.

\begin{restatable}{theorem}{thmvanillawhittle}
\label{thm:vanilla_whittle_canbe_arb_bad}
    For a \CBB{}, denote the \VanillaWhittle{} Policy's reward as $\Reward^\text{VanillaWhittle}$, and the optimal policy that satisfies context-specific budget constraint as $\Reward^\text{CBC-OPT}$. There exists an instance where, 
    \begin{equation*}
        \frac{\Reward^\text{CBC-OPT}}{\Reward^\text{VanillaWhittle}} \to \infty, \qquad \text{as}\,\, N \to \infty.
    \end{equation*}
\end{restatable}

\begin{proof}

Consider a \CBB{} instance with $N$ stochastically identical arms. For simplicity we assume that the transition probabilities are such that each arm is always active $(s_i = 1)$. 
Let there be two contexts, where context~$1$ occurs with probability $f_1 = 1 - \frac{1}{N}$ and context~$2$ occurs with probability $f_2 = \frac{1}{N}$. For each arm $i$, context~$1$ generates reward $r_i(s_i = 1, a_i = 1; k = 1) = \frac{1}{N}$, context~$2$ generates reward $r_i(s_i = 1, a_i = 1; k = 2) = N$. Suppose budget $B = 1$. 

Consider the policy that leaves all arms idle at context~$1$, and pulls all $N$ arms at context~$2$. The policy is feasible because its budget constraint $\vec B=(0, N)$ satisfies $f_1\times 0 + f_2 \times N= \frac 1N \times N =1=B$. Its average reward is $N$.

For \VanillaWhittle{}, the good context~$2$ that has the high reward only occurs with probability$\frac{1}{N}$. And when it happens, we can only pull and get reward from \textit{one} arm. So its average reward is $(1 - \frac{1}{N})\frac{1}{N} + \frac{1}{N}\times N = O(1)$. As $N\to \infty$, the gap between the two policies goes to infinity.
\end{proof}

\subsection{\ContextualOccupancyIndex{} and the \FlexibleBudgetAllocCWIndexPolicy{}}

Theorem~\ref{thm:vanilla_whittle_canbe_arb_bad} shows that the \VanillaWhittle{} Policy can be arbitrarily bad for \CBB{} compared to the optimal policy that satisfies the \ContextSpecificBudgetConstraint{}---we need new algorithms that can determine the optimal budget allocation $\vec B$ across contexts. In this section, we introduce the class of \FlexibleBudgetAllocCWIndexPolicy{}, which first determines and commit a budget allocation $\vec B$, and then pulls arms according to descending orders of \ContextualOccupancyIndex{}.

First, we introduce the \OMLP{} that is used to compute the \ContextualOccupancyIndex{}.
\begin{definition}
    The \emph{occupancy measure} $\mu$ of a (possibly randomized) policy $\pi$ in \CBB{} is the average visitation probability  to a state-action-context tuple $(s, a; k)$:
    $$
        \mu_i(s, a; k) := \Prx{s_i = s, a_i = a; k}, \forall i\in [N].
    $$
\end{definition}

We can formulate the problem of maximizing the stationary reward of \CBB{} as a linear program (LP) over occupancy measures:
\begin{definition}
\label{def:occ_measure_lp}
    For a given \CBB{} instance, its \emph{\OMLP{}} is
    \begin{align*}
        \max_{\mu} \ & \sum_{i\in [N]} \sum_{k\in [K]} \sum_{s_i, a_i} \mu_i(s_i, a_i; k) r_i(s_i, a_i; k) \tag{\OMLP}\\
        s.t.\  &f_{k'}\left[\sum_{k\in [K]} \sum_{s_i, a_i} P[s_i\to s_i'| a_i, k] \mu_i(s_i, a_i; k)\right]\\
        & \quad \quad \quad = \sum_{a_i}\mu_i(s_i', a_i; k'), \forall k', s_i', i\\
        & \sum_{k\in [K]} \sum_{a_i, s_i}\mu_i(s_i, a_i; k) = 1, \forall i\in [N] \\
        & \sum_{k\in [K]} \sum_i \sum_{s_i}\mu_i(s_i, 1; k)\le B \tag{Relaxed Budget Constraint}\\
        & \mu_i(s_i, a_i, k) \ge 0, \forall i, s_i, a_i, k .
    \end{align*}
\end{definition}

Solving the \OMLP{} induces a policy for \CBB{}, which pulls the arms according to optimal solution variables. We refer to it as the following \ContextualOccupancySoftBudgetPolicy{}:
\begin{definition}
[\ContextualOccupancySoftBudgetPolicy{} (adapted from~\citep{Xiong_2022_asymptotic_optimality})]
    Given the optimal solution $\mu^\star(\cdot, \cdot; k)$  to the \OMLP{}, the \emph{\ContextualOccupancySoftBudgetPolicy{}} $\pi^\text{soft}$ pulls an arm $i$ in state $s_i$ and context with probability $\chi^\star_i(s_i, k)$, where
    $$
        \chi^\star_i(s_i, k) = \frac{\mu_i^\star(s_i, 1; k)}{\mu_i^\star(s_i, 0; k) + \mu_i^\star(s_i, 1; k)}.
    $$
\end{definition}

The \ContextualOccupancySoftBudgetPolicy{} behaves as if it pulls all arms whose \ContextualOccupancyIndex{} defined below, are above a certain threshold.

\begin{definition}[Contextual Occupancy Index]
    For an arm $i\in [N]$ in a \CBB{} problem instance, solve its \OMLP{} for the instance and obtain its \emph{\ContextualOccupancyIndex{}} contingent with state and context as follows:
    \begin{align*}
        \rho_i(s_i, k) & := \chi_i^\star(s_i, k)r_i(s_i, a_i; k) \\
        & = \frac{\mu_i^\star(s_i, 1; k)}{\mu_i^\star(s_i, 0; k) + \mu_i^\star(s_i, 1; k)}r_i(s_i, a_i; k).
    \end{align*}
\end{definition}

Occupancy Index is a more robust index for RMAB compared to the Whittle Index. It behaves the same as \VanillaWhittle{} when the RMAB is indexable, but the asymptotic optimality for standard RMAB does not require indexability in general~(\citet{Xiong_2022_asymptotic_optimality}). Despite its many appealing properties, the \ContextualOccupancySoftBudgetPolicy{} does not satisfies the \ContextSpecificBudgetConstraint{} because \OMLP{}-Constraint III is a relaxed version---only the \emph{expected} number of arms pulled at each time point for a given context falls within the budget quantity. Therefore, the optimal value of \OMLP{}, which can be achieved by the \ContextualOccupancySoftBudgetPolicy{}, is an upperbound for all policies that satisfies \ContextSpecificBudgetConstraint{}:
\begin{align*}
    \Reward^\text{VanillaWhittle}\le \Reward^\text{CBC-OPT} \le \Reward^\text{LP}.
\end{align*}
Therefore, we define the class of \FlexibleBudgetAllocCWIndexPolicy{} as follows:
\begin{definition}
[\FlexibleBudgetAllocCWIndexPolicy]
    A \emph{\FlexibleBudgetAllocCWIndexPolicy{}} determines a budget allocation $\vec B \in \mathcal B_0 := \{ \vec{B} \in \mathbb{N}^K : \sum_{k\in [K]} f_k B_k \le B \}$
    given total budget $B$ and context probabilities $\vec f$, and pulls the $B_{k}$ arms with the highest \ContextualOccupancyIndex{}.
\end{definition}
We consider within the class of \FlexibleBudgetAllocCWIndexPolicy{} and focus on determining the optimal budget allocation $\vec B$. When there is no confusion, we write as $\Reward(\vec B)$ the reward of \CBB{} running \FlexibleBudgetAllocCWIndexPolicy{} with budget allocation $\vec B$ We aim at designing efficient algorithms that solve the following optimization problem:
\begin{align*}
    \max_{\vec B}\ & \Reward(\vec B) \\
    s.t.\ 
    & \sum_{k\in [K]} f_k B_k \le B.  \tag{Constraint II}\\
\end{align*}

\subsection{The \underline{\textbf{C}}ontextual \underline{\textbf{Occ}}upancy Index (\ContextualOccupancyIndexPolicy{}) Policy and its Suboptimality}

\ContextualOccupancySoftBudgetPolicy{} achieves higher reward than \VanillaWhittle{} because it shifts budget across time --- by saving up budget at bad context and using them when context is good. Guided by this insight, we design the Contextual Occupancy Index Policy (\ContextualOccupancyIndexPolicy{}) that mimics the \ContextualOccupancySoftBudgetPolicy{} but also satisfies \ContextSpecificBudgetConstraint{}. Specifically, it assigns each context the average budget that \ContextualOccupancySoftBudgetPolicy{} would have used for each context, and pulls arms with similar priority.
\begin{definition}[The \underline{\textbf{C}}ontextual \underline{\textbf{Occ}}upancy Index (\ContextualOccupancyIndexPolicy{}) Policy]
    \ContextualOccupancyIndexPolicy{} determines $\vec B$ by assigning the budget that \ContextualOccupancySoftBudgetPolicy{} would uses (on average) to each context, which can be obtained from the optimal solution of \OMLP{}:
    $$B_k = \frac1{f_k} \sum_{i, s_i} \mu^\star_i(s_i, 1; k), \quad \text{for} k\in [K].$$ 
    At each time step, if the context is $k$, it pulls top-$B_k$ arm
    s that have the highest positive \emph{\ContextualOccupancyIndex{}}.
\end{definition}

Standard RMAB problems and the \VanillaWhittle{} policy correspond to $B_{k}$ being the same across all $k$, and in such situation, the \ContextualOccupancyIndexPolicy{} is equivalent to the Whittle Index Policy, and is asymptotically optimal.\footnote{The asymptotic notion is usually to repeat all arms of an RMAB instance infinitely, along with the budget for the same repeats.}  However, while numerical studies shows that \ContextualOccupancyIndexPolicy{}'s performance is usually close to \OMLP{}'s upperbound, the following theorem shows that the \ContextualOccupancyIndexPolicy{} is not optimal for \CBB{}, with proof in Appendix~\ref{sec:proof_56}:

\begin{restatable}{theorem}{theoremFiveOverSix}
    \label{theorem:5/6}
    The \ContextualOccupancyIndexPolicy{}'s asymptotic approximation ratio compared to $\Reward^\text{CBC-OPT}$ is bounded above by $\frac56$.
\end{restatable}

The source of the suboptimality comes from when there are more than one contexts. The proof in Appendix~\ref{sec:proof_56} presents an original mathematical framework for asymptotic analysis.

\subsection{Finding the Optimal Budget Allocation}

Theorem~\ref{theorem:5/6} implies that \ContextualOccupancyIndexPolicy{} is suboptimal because it fails to determine the optimal budget allocation. In this section, we introduce the \BranchAndBound{} algorithm that finds the optimal budget allocation, and the \Mitosis{} algorithm that selects the near-optimal budget allocation in a no-regret fashion. 

In \CBB{}, the reward of \FlexibleBudgetAllocCWIndexPolicy{} under any budget allocation~$\vec B$ cannot be directly calculated from the \OMLP{}. We define the following methods for evaluating \CBB{}'s reward:
\begin{definition*}[\LP{} Upperbounds]
\ 

    \begin{itemize}
        \item
        For any budget allocation $\vec B$, denote as $\LP(\vec B)$ the optimal value of \OMLP{} with the following constraint inserted:
        \begin{equation*}
            \frac{1}{f_k}\sum_{i, s_i}\mu_i(s_i, 1; k) = B_k 
            \quad \forall B_k \in \vec{B}.
        \end{equation*}
        \item For any polytope region $\mathcal B \subseteq \mathcal B_0$, denote as $\LP(\mathcal B)$ the optimal value of $\max_{\vec B\in \mathcal B}\LP(\vec B)$.
    \end{itemize}
\end{definition*}

\begin{definition*}[Oracle Functions for $\vec B$]
    Fix any budget allocation $\vec B$, $\Reward(\vec B)$ can be evaluated by randomized simulation procedure that runs \FlexibleBudgetAllocCWIndexPolicy{} with $\vec{B}$ and returns the resulting reward of the \CBB{}. We define such as the following two types of \emph{oracle functions}:
    \begin{itemize}
        \item $\Oracle(\vec B)$: an oracle that returns an accurate estimate of $\Reward(\vec B)$ (eg. running many epochs of simulation).
        \item $\OracleSmall(\vec B)$: a fast oracle that returns a noisy estimate $\Reward(\vec B) + \epsilon$ (eg. running one epoch of simulation). Assume $\E \epsilon = 0$ and its variance is bounded.
    \end{itemize}
\end{definition*}

Notice that $\LP(\vec B) \ge \Oracle(\vec B)$, and $\LP(\mathcal B) \ge \max_{\vec B\in \mathcal B} \Oracle(\vec B)$, because \OMLP{} is a relaxation of the original problem. We design two algorithms that leverage this property search for the optimal budget allocation $\vec B^\star$ that maximizes $\Reward(\vec B)$.

\paragraph{The \BranchAndBound{} Algorithm} Using $\LP(\vec B) \ge \Oracle{}(\vec B), \forall \vec B$, this algorithm recursively splits the search region into smaller subregions and prunes subregions if its LP-based upperbound is lower than another's actual reward. 
Although \BranchAndBound{} is still NP-hard in the worst case, it provides a systematic way to efficiently search. We provide the pseudocode in Algorithm~\ref{alg:BnB}.

\begin{figure}[h]
    \centering
    \includegraphics[width=1\linewidth]{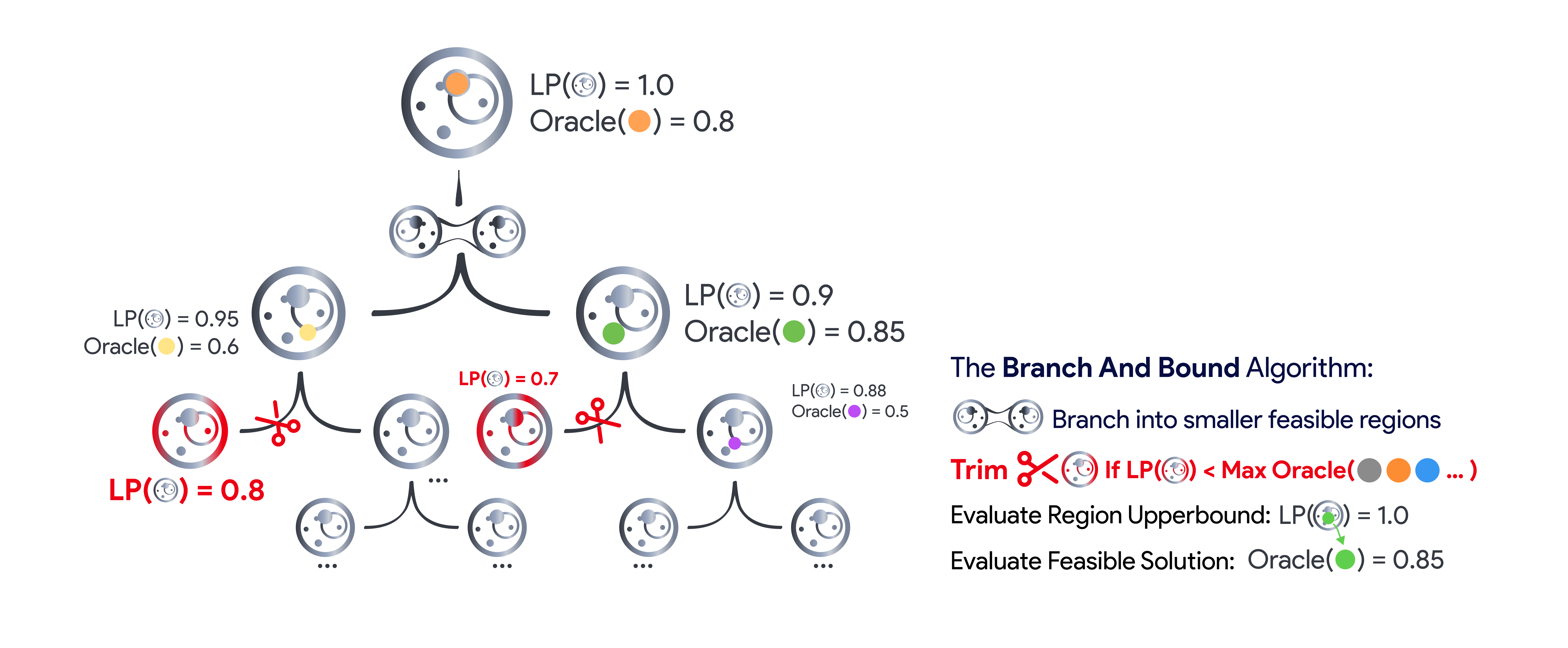}
    \caption{Visualization of \BranchAndBound{}}
\end{figure}

\begin{algorithm}[tb]
    \caption{\BranchAndBound{}}
    \label{alg:BnB}
    \textbf{Input}: Feasible Region $\mathcal B_0$, \LP{}, \Oracle{} \\
    \textbf{Output}: Budget Allocation $\vec{B}^*$
    \begin{algorithmic}[1]
        \STATE \textbf{Initialize:} $R^\text{OPT} \gets -\infty$, $\vec{B}^\text{OPT} \gets$ \texttt{None},  Queue $Q \gets \{\mathcal B_0\}$.
        \WHILE{$Q \neq \varnothing$}
            \STATE Dequeue $\mathcal B \gets Q.\texttt{pop}()$.
            \IF{$\LP(\mathcal B) < R^\text{OPT}$}
            \STATE \textbf{continue}.\COMMENT{prune $\mathcal B$}. 
            \ENDIF
            \STATE ${\vec B}^* \gets \vec B^\LP  (\mathcal B)$ \COMMENT{most promising $\vec B\in \mathcal B$}
            \IF{$\Oracle{}(\vec B^*) > L^*$}
                \STATE Update $L^* \gets\Oracle{}(\vec B^*)$ and $\vec{B}^\text{OPT} \gets \vec{B}^*$.
            \ENDIF
            \STATE \textbf{Branch:} Partition $\mathcal B = \mathcal B_1\cup \mathcal B_2$
            \STATE $Q \gets Q\cup \{\mathcal B_1, \mathcal B_2\}$.
        \ENDWHILE
        \STATE \textbf{return} $\vec{B}^*$.
    \end{algorithmic}
\end{algorithm}

\paragraph{The \Mitosis{} Algorithm} 
While \BranchAndBound{} already cuts down the search tree dramatically compared to brute force search, it still calls too many costly \Oracle{} evaluations on large-scale problems.
To address this issue, we develop the following multi-armed bandit (MAB) framework which allows for more nuanced speed-accuracy trade-off for the evaluation and search of better budget allocations:

\begin{definition}[Combinatorial MAB Framework for Solving \CBB{}]
    To find optimal $\vec B$, we define an associated \emph{Multi-Armed Bandit (MAB) problem} by identifying each arm with a vector $\vec{B} \in \mathcal{B}_0$. Pulling arm $\vec{B}$ invokes the fast oracle $\OracleSmall(\vec{B})$ which returns a noisy reward $\Reward(\vec{B}) + \epsilon$.
\end{definition}
For MAB, a standard UCB-type algorithm maintains empirical statistics for each arm and computes an index that serves as an upper confidence bound on the arm's true reward. For each arm (represented by $\vec{B}$) and time (in the MAB system) $\tau$, its upper-confidence-level index is given by
$
I_t(\vec{B}) := \hat{\mu}_t(\vec{B}) + f\bigl(N_t(\vec{B}), t\bigr),
$
where $N_t(\vec{B})$ is the number of times arm $\vec{B}$ has been selected and $\hat{\mu}_t(\vec{B})$ the empirical mean reward from $\vec{B}$.
$f$ is chosen so that, with high probability, $I_t(\vec{B})$ is an upper bound on the true mean reward $\mu(\vec{B})$. \footnote{For example, the classical UCB1 algorithm sets
$$
f\bigl(N_t(\vec{B}), t\bigr) = c \sqrt{\frac{\log t}{N_t(\vec{B})}},\text{ for }c> 0.
$$}

\paragraph{Addressing the Combinatorial Explosion with $\TreeArm$}
Naively applying UCB to our setting would result in combinatorial explosion, because the first step of UCB is to initialize all arms' empirical statistics by pulling each arm once. In \CBB{}, the number of arms is the number of integer solutions to $\sum_{k\in [K]} f_k B_k \le B$, which is $\mathcal{O}((BK)^K)$, so this initialization step is prohibitively expensive. To address this issue, we incorporate the hierarchical tree structure from \BranchAndBound{} to speed up our algorithms. 

We design $\TreeArm$ as a special arm that represents a \emph{group} of candidate budget allocations. Instead of tracking every $\vec{B}$, we use \TreeArm{} to encapsulate less-promising budget allocations, which are grouped in polytope regions~$\mathcal B_m \subseteq \mathcal B_0$. Formally,
\begin{definition}[\TreeArm{}]
    A \emph{\TreeArm{}} represents a union of polytopes subregions
    $
        \TreeArm{} = \cup_{m=1}^M \mathcal B_m.
    $
    When it is pulled, it splits out a most promising child arm:
    $
    \vec B_\text{new} := \argmax_{\vec B\in \TreeArm{}}\LP(\vec B),
    $
    and updates itself by excluding the new arm from itself ($\TreeArm \leftarrow \TreeArm \setminus \{\vec B_\text{new}\}$).

    The \TreeArm{}'s UCB index is set as $I_t(\TreeArm{}) := \max_{\vec B\in \TreeArm{}} \LP(\vec B)$.
\end{definition}

All arms $\vec B\in \TreeArm$ (before they are split out) will never be pulled. While these encapsulated arms have no empirical history for UCB value, , which upperbounds all encapsulated arms' rewards with probability $1$. For convenience, denote it as $\LP(\TreeArm)$. Note that because every pull of the \TreeArm{} splits out the child arm with the highest $\LP$ value, the \TreeArm{}'s UCB index $\LP(\TreeArm)$ decreases when it splits each time.

\Mitosis{} runs within the Combinatorial MAB Framework to find optimal budget allocation (see Algorithm~\ref{alg:Mitosis} for the pseudocode). It begins with the set of candidate arms containing only one $\TreeArm$, representing the entire feasible region $\mathcal{B}_0$. At each round, the algorithm selects from candidate arms (either a standard arm $\vec B$ or a $\TreeArm$) with the highest index. When a standard arm is pulled, we run $r_t(\vec B) \gets \OracleSmall{}(\vec B)$ to update its empirical statistics. When a $\TreeArm$ is selected, it splits out a new arm into candidate arms.  is named for how the \TreeArm{} ‘buds’ new arms progressively during the algorithm, similar to cell division in mitosis. Mitosis retains the no-regret guarantees of classical MAB algorithms. The proof is in Appendix~\ref{appendix:mitosis}.

\begin{restatable}{theorem}{thmMitosisNoRegret}
    [No-Regret of the Mitosis Algorithm]
    Let $\mathcal{A}$ denote the set of arms that have been pulled. After running the algorithm for $T$ rounds, the cumulative regret
    $
    R(T) \triangleq \sum_{t=1}^T \left( \mu^\star - \mu_{t} \right)
    $
    satisfies
    $$
    R(T) = \sum_{\vec{B}\in \mathcal{A}} \mathbb{E}[N_T(\vec{B})]\,\Delta(\vec{B}) = O\Biggl( \sum_{\vec{B}\in \mathcal{A}} \frac{\log T}{\Delta(\vec{B})} \Biggr),
    $$
    which matches the UCB1 regret bound.
\end{restatable}

\begin{remark*}
    The reason that cannot directly apply classical MAB algorithms is \ContextSpecificBudgetConstraint{} has combinatorially many integer solutions, making it infeasible to evaluate every solution even once for no-regret UCB-type algorithms to initialize empirical statistics . We pull inspiration from \BranchAndBound{}'s tree structure to design the \TreeArm{} that encapsulates many arms, accommodated by \CBB{}'s \OMLP{} which provides a strong upperbound for the \TreeArm{}'s UCB index. By marrying the no-regret MAB algorithms with \CBB{}'s special solution space structure, Mitosis leverages the best of both worlds: it efficiently search through the combinatorial solution space under \OracleSmall{}'s fast but noisy reward estimates.
\end{remark*}

\begin{figure}[h]
    \centering
    \includegraphics[width=1\linewidth]{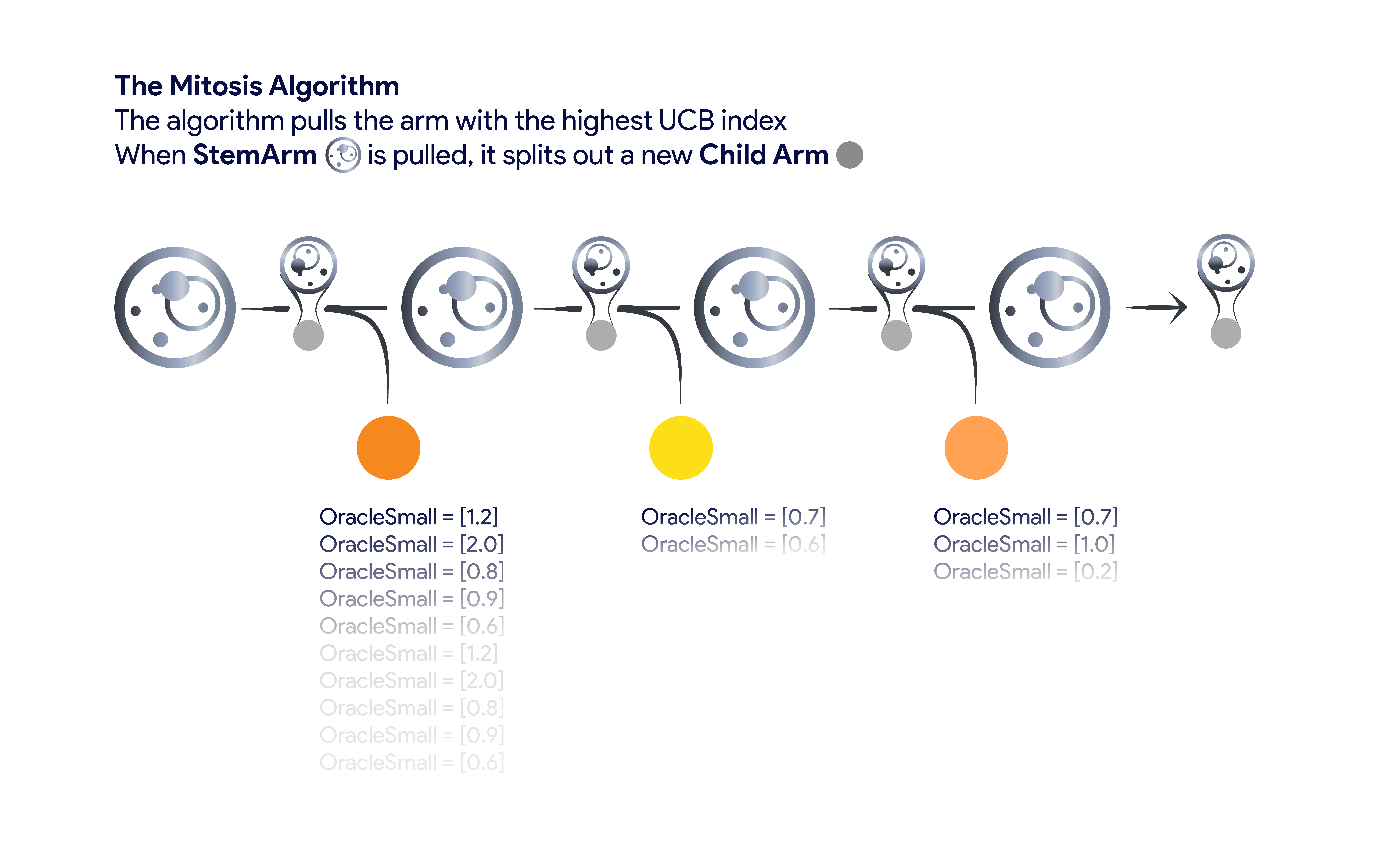}
    \caption{Visualization of \Mitosis{}}
\end{figure}

\begin{algorithm}[tb]
    \caption{\Mitosis{}}
    \label{alg:Mitosis}
    \textbf{Input:} Feasible region \(\mathcal{B}_0\), LP upper-bound function \(\LP\), fast oracle \(\OracleSmall\) \\
    \textbf{Auxiliary:} UCB-type no-regret algorithm $I_{t}(\cdot)$ \\
    \textbf{Output:} Budget allocation \(\vec{B}^*\) with high empirical reward
    \begin{algorithmic}[1]
        \STATE \textbf{Initialize:}
            \begin{itemize}
                \item Initialize $\TreeArm{}:=\{\mathcal B_0\}$.
                \item Candidate set (heap) \(\mathcal{A} \gets \{\TreeArm\}\).
                \item Set time \(t \gets 0\).
            \end{itemize}
        \WHILE{\(t < T\) \textbf{and} stopping condition not met}
            \STATE Select from candidate set w.r.t. UCB index:
            \[
                a^* \gets \argmax_{a \in \mathcal{A}} I_t(a),
            \]
            \IF{\(a^*\) is a \textbf{standard arm} (i.e., corresponds to a specific \(\vec{B}\))}
                \STATE Pull arm and observe reward $r_t(a^*)\gets \OracleSmall{}(\vec B)$.
                \STATE Update the empirical statistics \(N_t(a^*)\) and \(\hat{\mu}_t(a^*)\).
            \ELSE
                \STATE \textbf{// \(a^*\) is a \(\TreeArm\); splits out and pull new arm.}
                \STATE \(\TreeArm\) splits arm \(a^*\) to obtain a new standard arm \(a'\) with allocation \(\vec{B}_{a'} = \vec{B}_\text{new}\).
                \STATE Insert \(a'\) into the candidate set: \(\mathcal{A} \gets \mathcal{A} \cup \{a'\}\).
            \ENDIF
            \STATE \(t \gets t + 1\), update the candidate set \(\mathcal{A}\) with $I_{t+1}(\cdot)$.
        \ENDWHILE
        \STATE \textbf{return} The allocation \(\vec{B}^*\) corresponding to the arm in \(\mathcal{A}\) with the highest empirical mean reward.
    \end{algorithmic}
\end{algorithm}
\section{Fairness in \CBB{}}
\label{sec:fairness}

\begin{definition}
    Define fairness index for \textit{any} \CBB{} policy~$\pi$ as follows\footnote{Assume that $f_k > 0$ for all $k\in [K]$.}:
    \begin{equation*}
      \Fairness(\pi) : = \min_{k\in [K]} \frac1{f_k}\frac{\lim_{T\to \infty} \E_{(a, s)\sim \pi, k_t\sim \mathcal F}[\sum_{t = 1}^T \sum_{i\in [N]} r_i(s_i, a_i; k_t)\mathbb I\{k_t = k\}]}{\lim_{T\to \infty} \E_{(a, s)\sim \pi, k_t\sim \mathcal F}[\sum_{t = 1}^T \sum_{i\in [N]} r_i(s_i, a_i; k_t)]}.
    \end{equation*}
\end{definition}

This definition captures the idea that a fair algorithm should achieve reward for each context in proportion to the context's frequency. Or, it can also be understood as minimum over the average reward of each context divided by total average reward. Observe that $\Fairness \in [0,1 ]$ by construction. $\Fairness = 1$ indicates perfect fairness, $\Fairness = 0$ indicates extreme unfairness, where at least one context receives no reward regardless of its occurance frequency. In requirement of fairness within the \FlexibleBudgetAllocCWIndexPolicy{} class, we consider the following optimization:
\begin{align*}
    \max_{\vec B}\ & \Reward(\vec B) \\
    s.t.\ 
    & \sum_{k\in [K]} f_k B_k \le B  \tag{Constraint II}\\
    & \Fairness(\vec B) \ge \theta \tag{Constraint for Fairness}
\end{align*}

\paragraph{Modifying \ContextualOccupancyIndexPolicy{}, \BranchAndBound{} and \Mitosis{} for Fairness}
We obtain a fair version of \ContextualOccupancyIndexPolicy{} by adding the following \emph{Linear Constraint for Fairness} to the \OMLP{}:
\begin{align}
  \underbrace{\frac 1{f_k} {\sum_{i\in [N]} \sum_{s_i,a_i}\mu_i(s_i,a_i;k)r_i(s_i,a_i;k)}}_\text{reward for type $k$}\ge \theta& \underbrace{\left(\sum_{i\in [N]}\sum_{k\in [K]}\sum_{s_i,a_i}\mu_i(s_i,a_i;k)r_i(s_i,a_i;k)\right)}_\text{total reward} ,\quad \forall k\in [K] \tag{Linear Constraint for Fairness}
\end{align}

\BranchAndBound{} and \Mitosis{} can be directly modified to incorporate the fairness constraint by adjusting the $\LP{} upperbounds$ and the oracle functions. Let the fairness requirement be $\theta \in [0, 1]$. We define the following methods for evaluating \CBB{}'s reward under fairness constraint:
\begin{definition*}[Fairness-aware \LP{} Upperbounds and Oracles]
  For any budget allocation $\vec B$, polytope region $\mathcal B$ and fairness requirement $\theta$, we extend the definitions of \LP{} upperbounds and oracles as follows:
\begin{itemize}
    \item $\LP_\text{fair}(\vec B)$: the optimal value of the \OMLP{} with additional Linear Constraint for Fairness inserted.
    \item $\LP_\text{fair}(\mathcal B)$: the maximum of $\LP_\text{fair}(\vec B)$ over all $\vec B\in \mathcal B$.
    \item 
    $
    \Oracle_\text{fair}(\vec B) := \Oracle_\text{fair}(\vec B) \times \mathbb I\{\Fairness(\vec B) \ge \theta\},
    $ where $\Fairness(\vec B)$ is the fairness index achieved by \FlexibleBudgetAllocCWIndexPolicy{} with budget allocation $\vec B$.
    \item $\OracleSmall_\text{fair}(\vec B) := \OracleSmall_\text{fair}(\vec B) \times \mathbb I\{\Fairness(\vec B) \ge \theta\}$. Similarly, $\Fairness(\vec B)$ is the fairness index achieved by \FlexibleBudgetAllocCWIndexPolicy{} with budget allocation $\vec B$.
  \end{itemize}
\end{definition*}
By expanding the definitions of $\LP(\cdot)$, $\Oracle(\cdot)$ and $\OracleSmall(\cdot)$ to their fairness-aware versions, we can directly apply \BranchAndBound{} and \Mitosis{} to solve the optimal budget allocation under fairness constraint.
\section{Experiments}

In this section, we empirically evaluate the performance of \ContextualOccupancyIndexPolicy{}, \BranchAndBound{} and \Mitosis{} in \CBB{}, and compare them against three baseline algorithms: \textit{\Random{}}, \textit{\Greedy{}} and the \textit{\VanillaWhittle{}} Policy.
We first demonstrate the algorithms' performance, comparing their reward and runtime on numerical simulations. Then, we show that some of these algorithms could be sensitive to problem instance parameters, while our \Mitosis{} algorithm consistently performs the best, robust against all these different setups that could arise in the real world.
Finally, we run \CBB{} experiment on real-world food rescue data which further confirms the superiority of \Mitosis{}.
For a coherent presentation and concrete argument, we describe all experiments in the food rescue context. That said, the results in Section~\ref{sec:synthetic_food_rescue} and~\ref{sec:sensitivity_experiments} are evidently generalizable beyond the food rescue setting.
All experiments are run on an AMD Ryzen 5955WX CPU with 128GB RAM.

\subsection{Experiments on Synthetic Data}
\label{sec:synthetic_food_rescue}

\paragraph{Food Rescue \CBB{} Setup} Consider a food rescue platform notifying volunteers to pick up food donation delivery tasks. There are $K$ regions. Each region $k\in[K]$ is modeled as a context, and is associated with a \textit{popularity} index $\mathrm{pop}_k\in \Re$. There are $N$ volunteers. Each $i\in [N]$ volunteer is modeled as an arm. Every volunteer has a historical rescue record vector $H_i = (k_1, k_2, \ldots), k_j\in [K]$ that records which region's tasks they have taken. Each volunteer or region is endowed with a location $(x, y)$. Depending on whether we run experiments on synthetic data or real-world data, we obtain these attributes either by sampling from random distributions or directly from food rescue data.

At each time step $t\le T$, a food rescue trip arises from one region $k\in [K]$ chosen independently with probability $f_k$ ($\sum_{k\in [K]}f_k = 1$).
The decision is to notify some volunteers via action $a_i^t =1$, subject to the \ContextSpecificBudgetConstraint{}.
Volunteers' state space is binary: \textit{active}~($s = 1$) or \textit{inactive}~($s = 0$). Only by notifying active volunteers the platform obtains reward. Reward and state-action dependent transition probabilities are contingent on regions' and volunteers' attributes and distance. 
We introduce the high-level intuition here and defer the details to Appendix~\ref{appendix:activeness_details}.
We model two types of food rescue volunteers. There are ``organic'' volunteers, whose reward and transition dynamics are a function of their distance from the region and their historical activities. Then there are ``churner'' volunteers, who could have a high likelihood of claiming a trip in some regions, but then slide to inactive states and very hard for them to become active again. In the food rescue application, as is true in many applications, both types of users exist. It is thus crucial to model them separately. For now, we assume both types of volunteers appear equally frequently in the volunteer population. In Section~\ref{sec:sensitivity_experiments}, we will relax this assumption.

\paragraph{Experiment Setup and Results}

\begin{figure}
    \centering
    \includegraphics[width=0.66\linewidth]{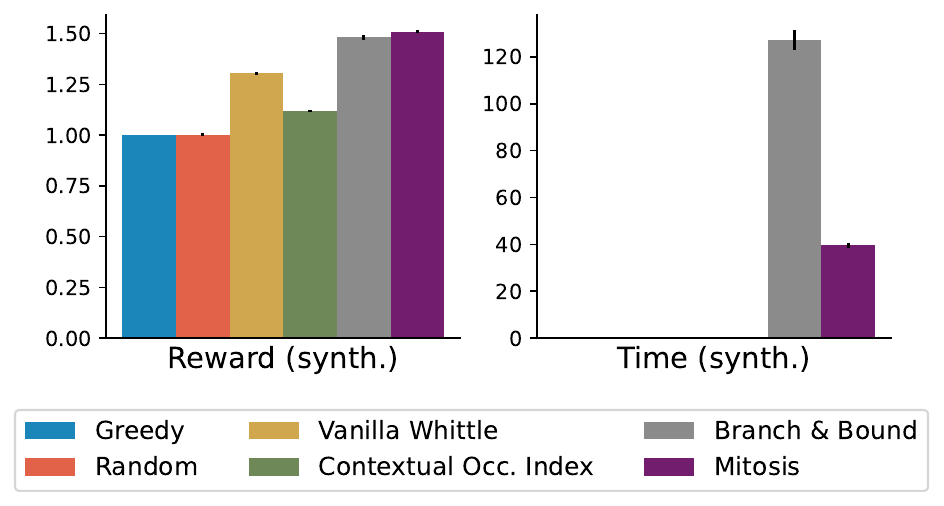}
    \caption{Main synthetic experiment with 50 arms, 3 contexts, and 0.1 budget ratio. Figures show normalized mean reward (left) and runtime in seconds (right) averaged over 32 seeds. \Mitosis{} yields the highest reward at significantly lower computation than~\BranchAndBound{}.
}
    \label{fig:main_synthetic}
\end{figure}

In a synthetic \CBB{} instance (50 arms, 3 contexts, budget=5), we run 32 seeds, each with 100 trials. We compare our proposed algorithms (\ContextualOccupancyIndexPolicy{} and \Mitosis{}) with the following benchmarks: \textit{\Random{}} policy that selects arms uniformly at random; \textit{\Greedy{}} policy that selects arms with the highest immediate reward $r_i^k(s_i^t,1)$, the aforementioned \textit{\VanillaWhittle{}} policy that is asymptotically optimal in standard RMAB, and \textit{\BranchAndBound{}} that is guaranteed to compute optimal budget allocation but is computationally expensive. 

Bar plots~(Figure~\ref{fig:main_synthetic}) compare the reward of each algorithm normalized by \Random{}'s, and runtime measured in seconds. \Mitosis{} (purple) achieves the highest reward overall, while \BranchAndBound{} (gray) is almost as good, but much slower. In fact, we set a timeout limit and \BranchAndBound{} often terminates before it finds the optimal reward. The simpler baselines (\Greedy, \Random, \VanillaWhittle, and \ContextualOccupancyIndexPolicy{}), though taking almost no time to initialize, provide moderate to lower rewards, with the \ContextualOccupancyIndexPolicy{} notably underperforming.

\begin{figure}
    \centering
    \includegraphics[width=0.88\linewidth]{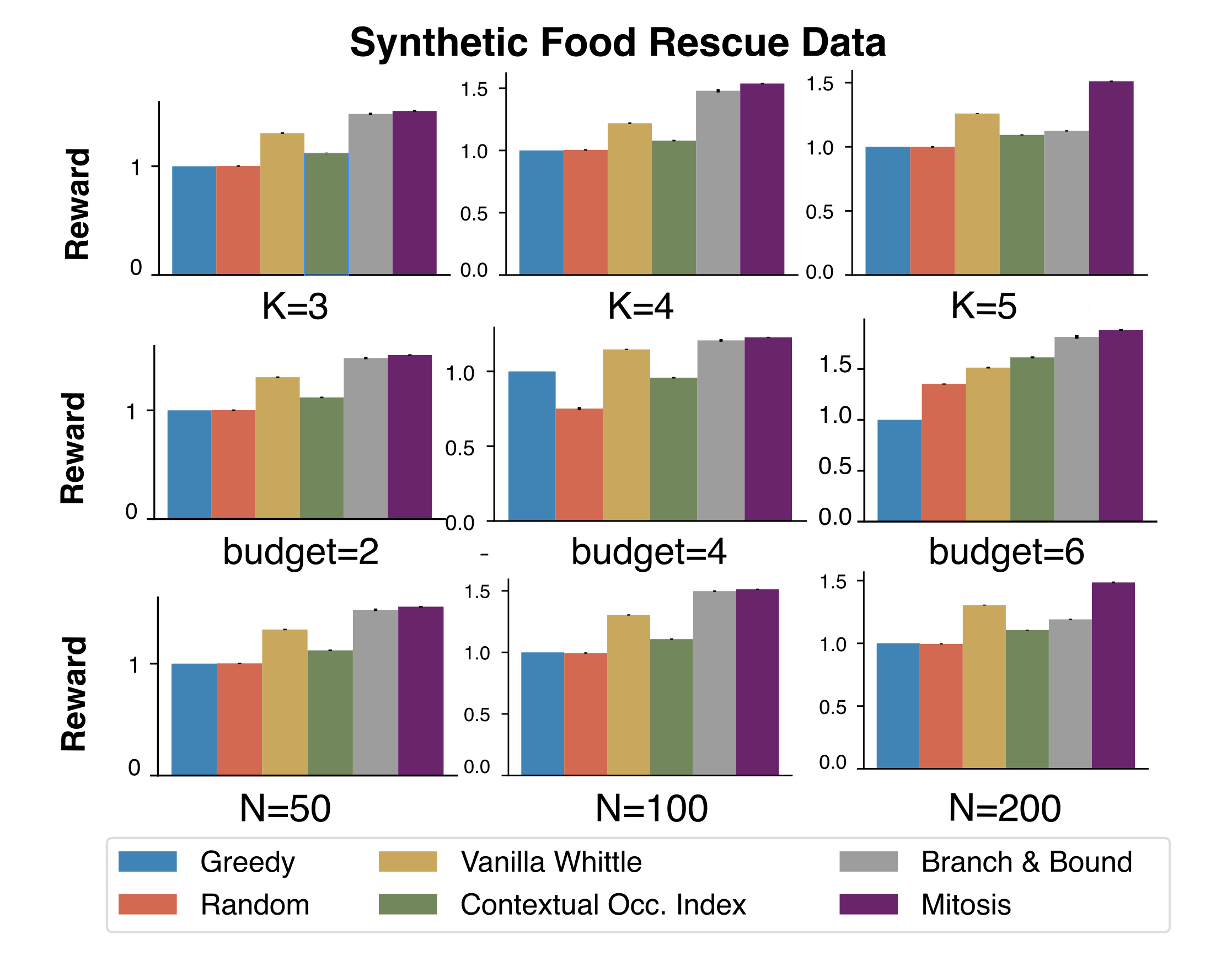}
    \caption{Ablation studies of the synthetic food rescue experiments to Figure~\ref{fig:main_synthetic}. \Mitosis{} generally performs the best across different problem sizes.
}
    \label{fig:main_synthetic_ablations}
\end{figure}

We also run ablation studies by varying the number of volunteers $(N = 50, 100, 200)$, the number of regions $(K = 3, 4, 5)$ and budgets $(B = 2, 4, 6)$. 
Figure~\ref{fig:main_synthetic_ablations} shows the rewards of the various algorithms when varying $N$ and $B$. Due to page limit, we defer the time plots and other reward plots to Appendix~\ref{appendix:experiment_details}.
Generally, similar performance pattern holds as the scale of the instance increases. \Mitosis{} consistently performs the best, while \BranchAndBound{}'s performance drops due to timeout. 

\subsection{Sensitivity Analysis}
\label{sec:sensitivity_experiments}

\begin{figure}[t]
    \centering
   \includegraphics[width=0.66\linewidth]{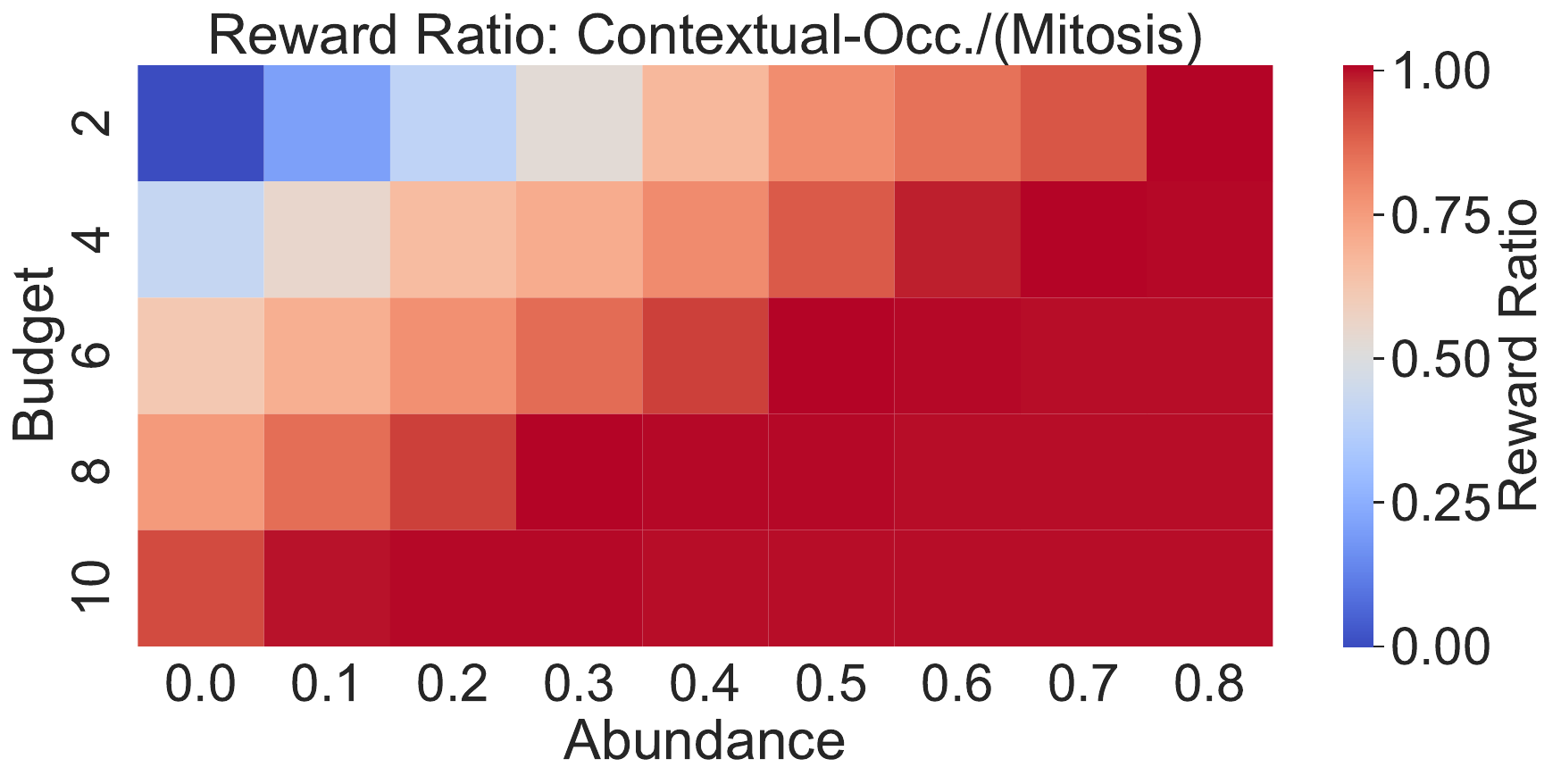}
    \caption{Heatmap showing the ratio between \ContextualOccupancyIndexPolicy{} and \ZoomingBranch{}'s reward. \Mitosis{} has a performance edge over \ContextualOccupancyIndexPolicy{} as the proportion of churner volunteers increases and as the notification budget decreases. 
}
    \label{fig:heatmap}
\end{figure}

While the ablation results in Figure~\ref{fig:main_synthetic_ablations} demonstrate the robustness of the \Mitosis{} algorithm, some questions remain. Namely, we assumed that there are two types of volunteers -- organic and churner -- and they appear equally likely. In reality, for different applications, the ratio between organic and churner users may vary greatly. 
In this section, we aim to paint a complete picture of the algorithms' performance under various conditions.

To begin with, we refer to the proportion of organic volunteers in the volunteer population as ``abundance''.
We systematically vary the abundance and budget~$B$ in food rescue \CBB{}, and and plot the heatmap of the ratio between \ContextualOccupancyIndexPolicy{}'s and \Mitosis{}'s rewards in Figure~\ref{fig:heatmap}. 
As shown in the figure, as abundance increases, the reward gap between \Mitosis{} and \ContextualOccupancyIndexPolicy{} gradually closes, and similarly when the notification budget increases. However, if we have more churner volunteers, or if we have small notification budget relative to the overall number of volunteers, the performance of \ContextualOccupancyIndexPolicy{} could suffer. For churners, their different attitudes (transition dynamics) towards different regions (contexts) lead the \ContextualOccupancyIndexPolicy{} algorithm to make mistakes in action selection.

\paragraph{Completely Random \CBB{}.}
To further stress test our main algorithm \Mitosis{}'s performance in settings most adversarial to it, we run experiments on the following simple setup with completely randomly generated \CBB{} instances.
    Let transition probabilities $\Pr[s^{t + 1} \mid s , a; k]$ and rewards $r(s, a; k)$ be generated from the following distributions:~\footnote{We use a common modeling assumption where rewards accrue only from engaged and available arms~\citep{Zhao2020NetworkMABBook}.}
    
    \begin{align*}
    & \Pr[s^{t + 1}_i = 1\mid s , a; k] \sim \mathrm{Clip}(\mathrm{Normal}(\mu^{(s, a)}_k, \sigma^{(s, a)}_k), 0, 1)\\
    & \Pr[s_i^{t + 1} = 0\mid s, a; k] = 1- \Pr[s_i^{t + 1} = 1\mid s, a; k]\\
    & r_{i}(s=1, a=1; k) \sim \mathrm{Normal}(\mu^r_k, \sigma^r_k)\\
    & r_i(s, a;k) = 0 \text{ otherwise.}
    \end{align*}
    where $\mu_k^{(s, a)}, \sigma_k^{(s, a)}$ and $\mu_k^{r}, \sigma_k^{r}, \;\forall k, s, a$ are all sampled i.i.d. from $\Uniform[0, 1]$. To guarantee indexability, we further enforce $\Pr[s^{t + 1}_i = 1\mid s = 1, a = 1; k] < \Pr[s^{t + 1}_i = 1\mid s = 1, a = 0; k]$ (fatigue) and $\Pr[s^{t + 1}_i = 1\mid s = 0, a = 1; k] > \Pr[s^{t + 1}_i = 1\mid s = 0, a = 0; k]$ (recovery).
    Context distributions over $[K]$ is defined by sampling weights ${w}_k \sim \text{Uni}[0, 1]$ and normalize $f_k = \frac{w_k}{\sum_{\kappa}w_\kappa}, \forall k\in [K]$.
    With this setup, we have ripped off almost all the real-world relevance and treat this as a pure mathematical model.

\begin{figure}
    \centering
    \includegraphics[width=0.66\linewidth]{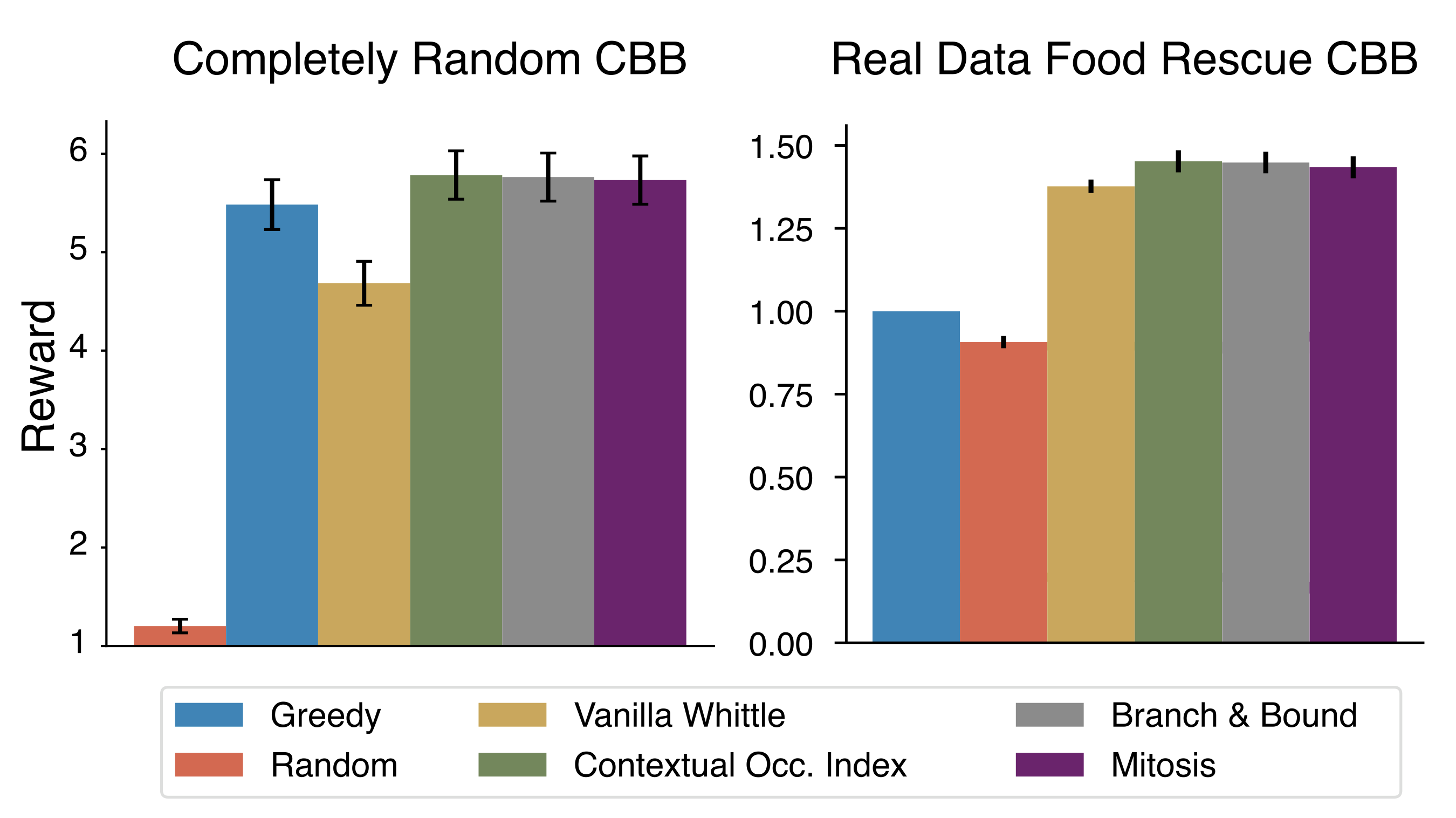}
    \caption{Reward of different policies for Completely Random \CBB{} (Left) and Real-data Food Rescue \CBB{} (Right)}
    \label{fig:completely_generic_barplot}
    \label{fig:main_real}
\end{figure}

We run the experiment with 50 arms, 5 contexts, and a notification budget of 5. We run 32 seeds with 100 trials each. The bar plot in Figure~\ref{fig:completely_generic_barplot} compares the reward and runtime of \Mitosis, \ContextualOccupancyIndexPolicy{} and other aforementioned baselines.
In Figure~\ref{fig:completely_generic_barplot}, context-agnostic \VanillaWhittle{} policy does not perform well, which validates Theorem~\ref{thm:vanilla_whittle_canbe_arb_bad} that \VanillaWhittle{} can perform arbitrarily poorly in the worst case for \CBB{}. 
Meanwhile, context-aware policies -- \Mitosis{}, \ContextualOccupancyIndexPolicy{} and \BranchAndBound{} -- perform equally best. This is as expected since the instance sampling parameters are homogeneous and resemble the organic volunteer case above.

With this, we arrive at the main takeaway from Sections~\ref{sec:synthetic_food_rescue} and~\ref{sec:sensitivity_experiments}. \ContextualOccupancyIndexPolicy{} generally outperforms the baselines, and performs optimally in some settings. However, it fails in some other settings. On the other hand, \Mitosis{} and \BranchAndBound{} perform optimally in all settings, as the theoretical guarantee suggests. However, \Mitosis{} is much more computationally efficient than \BranchAndBound{}, making it a clear winner of all.

\subsection{Experiment with Real Food Rescue Data}
\label{sec:real_experiments}
We construct food rescue \CBB{} from real data by sampling from a total pool of more than 500 thousand volunteers. Volunteers’ and regions’ attributes (locations and other idiosyncratic factors) are obtained from real-world data. The experiment setup is similar as the synthetic experiments in section~\ref{sec:synthetic_food_rescue} (same set of policies, N=50, K=3, Budget=5, 32 seeds with 100 trials per seed). Results are shown in the barplot in Figure~\ref{fig:main_real}.

On real data, the context‐aware methods (\ContextualOccupancyIndexPolicy{}, \BranchAndBound{} and \Mitosis{}) outperform \Greedy{}, \Random{} and \VanillaWhittle{}. \BranchAndBound{} yields the highest average reward but requires disproportionately longer runtimes. By contrast, \Mitosis{} nearly matches \BranchAndBound{} while significantly reducing computation. Notably, \ContextualOccupancyIndexPolicy{} catches up with \Mitosis{}---confirming it benefits from real‐world attribute structure. The policies' performance trend is similar when we vary the number of volunteers, the number of regions and budget level in the ablation studies (see Appendix~\ref{appendix:experiment_details} for details). This implies in application, \ContextualOccupancyIndexPolicy{} is sufficient for near-optimal performance. \Mitosis{} guarantees optimality and is significantly faster than \BranchAndBound{}.

\subsection{Fairness in Practice}

Food rescue organizations want to ensure equitable volunteer engagement across regions. We evaluate the fairness of different policies in food rescue \CBB{} using the fairness metric defined in Section~\ref{sec:fairness}. Figure~\ref{fig:pareto_frontier} shows the Pareto frontier of \Reward versus \Fairness, illustrating the inherent trade-off between the two objectives. Among fairness-aware policies, \Mitosis{} consistently achieves the highest rewards across varying fairness levels. \BranchAndBound{} performs comparably but incurs substantially higher computational costs. While \ContextualOccupancyIndexPolicy{} has reasonable performance, it is consistently outperformed by both \Mitosis{} and \BranchAndBound{}, especially under stricter fairness requirements. These results suggest that \Mitosis{} is a strong candidate for food rescue platforms seeking to balance reward and fairness.

\begin{figure}[t]
\centering
\includegraphics[width=0.5\linewidth]{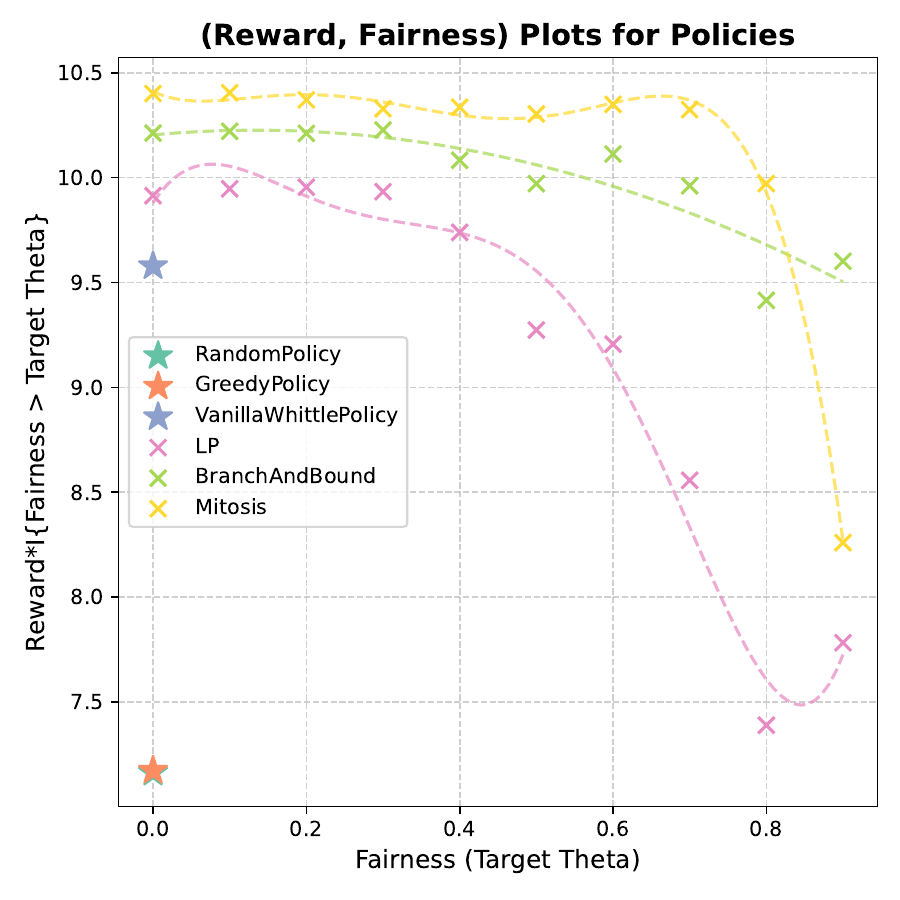}
\caption{Pareto frontier of reward versus fairness.}
\label{fig:pareto_frontier}
\end{figure}

For a representative real-world \CBB{} instance with 1,000 volunteers, three regions, and a budget of 200, computational efficiency becomes a critical consideration in this large-scale deployment. \Mitosis{} fails to terminate within 12 hours, whereas \ContextualOccupancyIndexPolicy{} completes in under a minute. This efficiency makes \ContextualOccupancyIndexPolicy{} an attractive choice for platforms seeking a practical balance among fairness, reward, and runtime.

\begin{figure}[t]
\centering
\includegraphics[width=1\linewidth]{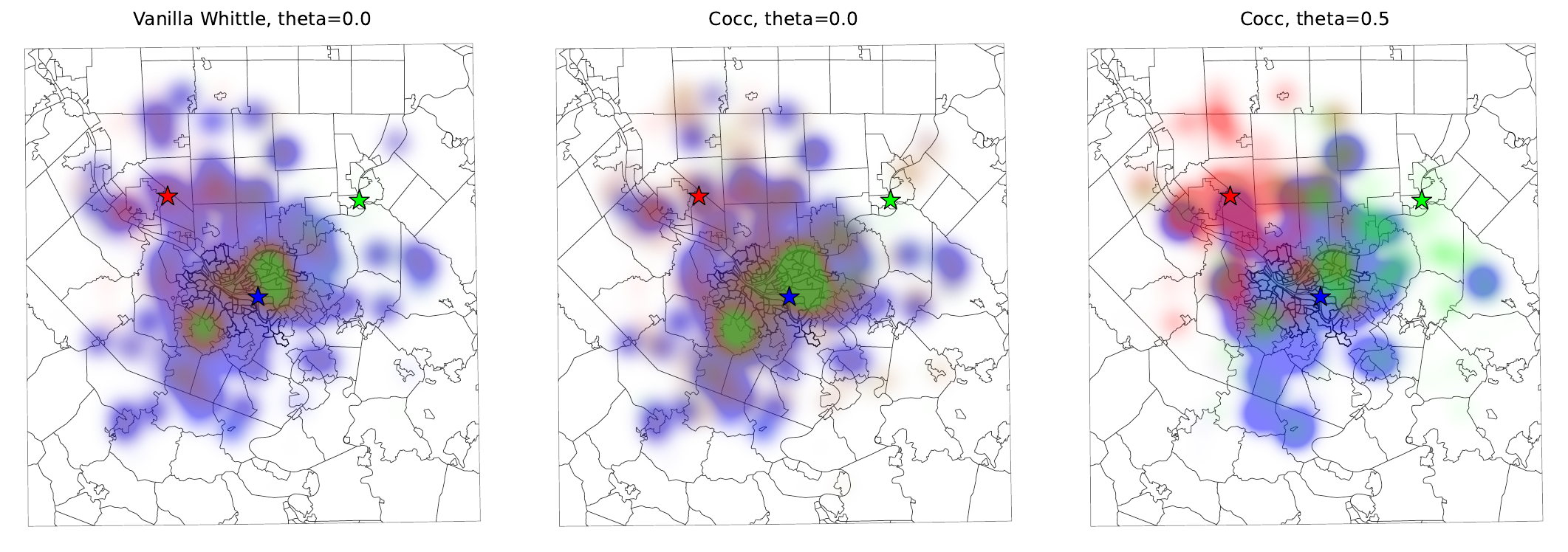}
\caption{Volunteer allocation under different policies for three Pittsburgh-area regions. Colors indicate the frequency of volunteer matches with each region. Without fairness constraints, the central downtown region (blue) dominates allocation; with fairness requirements, volunteer support shifts toward suburban regions, especially the under-served northwestern region (red).}
\label{fig:food_rescue_map}
\end{figure}

Figure~\ref{fig:food_rescue_map} visualizes volunteer allocations for three Pittsburgh-area regions under different policies. The central ``popular'' downtown region (blue) has 720 nearby volunteers clustered with an average distance of 7.7 km to their rescue spots\footnote{Volunteers are clustered based on proximity to their closest rescue location.}. In contrast, the least favored northwestern region (red) has only 85 volunteers on average and a longer average distance of 10.6 km.

The figure highlights a striking disparity in volunteer allocations. Under \VanillaWhittle{} or \ContextualOccupancyIndexPolicy{} with no fairness requirement, the central downtown region (distincted in blue) dominates volunteer engagement, receiving disproportionately high support due to its density and accessibility. However, as fairness constraints increase (e.g., 50\% fairness requirement), the allocation becomes more balanced: the under-served suburban regions, particularly the northwestern  region (in red), receive substantially more volunteer support. This illustrates how fairness-aware policies can effectively correct structural imbalances in volunteer distribution while retaining competitive reward performance.

\section*{Acknowledgement}
We thank~\frname{} for thoughtful discussions and continued support of this work.
Co-author Tang would like to thank the faculty and students of the Institute of Theoretical Computer Science at Shanghai University of Finance and Economics for their support and valuable comments.
Co-author Raman is also supported by an NSF GRFP Fellowship.
Co-author Shi is partially supported by a Google Academic Research Award and a John C. Mascaro Faculty Fellowship.

\clearpage
\bibliographystyle{named}
\bibliography{bibFiles/fairness_papers,%
              bibFiles/RMAB_papers,%
              bibFiles/contextualRMAB}

\clearpage
\appendix

\section{Proof of Theorem~\ref{theorem:5/6}}
\label{sec:proof_56}
\theoremFiveOverSix*
 \begin{proof}

    \textbf{Outline} First, we formally establish the asymptotic framework, which is where the Whittle Index Policy for standard RMAB achieves optimality. Then we introduce how to analyze \CBB{}'s in this asymptotic regime. Finally, we present the instance where the $\frac 56$ bound is achieved.

\subsection*{Asymptotic Notion}
    
    We define the asymptotic notion for analyzing (sub-)optimality for \CBB{}. It is same to the approach for standard RMAB, originally proposed by ~\cite{Weber_Weiss_1990} for RMAB with stochastically identical arms and generalized to heterogeneous arms by~\cite{Xiong_2022_asymptotic_optimality}:
    \begin{definition}[$\rho$-scaled \CBB{}]
    Fix a \emph{Base} \CBB{} instance with $M$ arms
    $$\langle M, \mathcal{S}, \mathcal{A}, K, \lbrace r_{i}^{k} \rbrace_{i \in [M], k \in \mathcal{K}}, \lbrace P_{i}^{k} \rbrace_{i \in [M], k \in \mathcal{K}}, \mathcal{F}\rangle.$$
    With budget $B\in \mathbb N$.

    Now, consider each arm being replicated $\rho$ times, with the budget scaled by $\rho$ as well. The new \CBB{} instance has $\rho \times M$ arms, with each of the $M$ arms in the base \CBB{} repeated $\rho$ times. Budget is scaled to $\rho B$.
\end{definition}

For a base contextual RMAB instance scaled with $\rho$, when there is no confusion about the base instance we're referring to, denote its reward for any policy $\pi$ as
$$
\Reward^\pi(\rho) := \lim_{T \to \infty} \mathbb{E}_{\vec {a} \sim \pi(\cdot), k\sim \mathcal F}\Bigl[\frac{1}{T} \sum_{t=1}^T \sum_{i \in [\rho M]} r_i^k(s_i^t, a_i^t)\Bigr].
$$

For $\rho$-scaled \CBB{}, notice that its reward upperbound from solving the \OMLP{} simply scales with $\rho$: 
$$
\Reward^\text{LP}(\rho) = \rho \overline{\Reward}(1).
$$

We refer to every arm $i\in [M]$ in the base instance as a \textbf{type-$i$} arm, and its $\rho$ replicates in the $\rho$-scaled \CBB{} as the $\rho$ type-$i$ arms.

\subsection*{Asymptotic System Behavior for \CBB{}}

In the following section we introduce a new method for analyzing the asymptotic behavior of \CBB{} as $\rho\to\infty$. It is different from the standard approach of~\cite{Weber_Weiss_1990}. 

To ease the complication of notations, we describe our method with \CBB{} that $r_i(s_i=0, a=0; k) = r_i(s_i=0, a=1; k)=0, \forall i, k$, and transition probabilities $P_i^ks_i^{t+1}\mid s_i^{t}, a_i^t=1] = P_i^ks_i^{t+1}\mid s_i^{t}, a_i^t=0],\forall s_i^{t+1}, s_i^t$. In this way it is meaningless to pull inactive arms, since it makes no difference in rewards nor transition probabilities. Generalization to general \CBB{} is without loss of generality.

We care about the proportion of active arms as $\rho\to \infty$ of the $\rho$. The following technical lemma characterizes the dynamic of arms:
\begin{lemma}
\label{lem:direc_delta}
    Denote as $a_i^t$ the proportion of type-$i$ active arms at any time point $t$ under given policy $\pi$. Conditional on $a_i^t$ and context $k$, $a_i^{t + 1}$'s distribution converges to a Direc Delta function $\delta_\text{shift}(\cdot) $, shifted with $\mathbb E[a_i^{t + 1} \mid a_i^t, k]$ as the total number of arms $\rho\to \infty$. In other words,
    \begin{align}
        f(a_i^{t + 1} \mid a_i^t, k) & = \delta_
        {\mathbb E[a_i^{t + 1} \mid a_i^t, k]}(a_i^{t +1 }),
    \end{align}
    where, let $B_{i, k}$ be the number of active type-$i$ arms pulled by the policy at context $k$:
    \begin{align}
        \mathbb E[a_i^{t + 1} \mid a_i^t, k] & = \min(a_i^t, \frac{B_{i, k}}\rho)P_i^ks_i^{t+1} =1 \mid s_i^t=1, a_i^t=1] \\
     & + (a_i^t - \min(a_i^t, \frac{B_{i, k}}\rho)\\
     & \quad \times P_i^ks_i^{t+1} =0 \mid s_i^t=1, a_i^t=1] \\
     & + (1 - a_i^t)P_i^ks_i^{t+1} =1 \mid s_i^t=0]
    \end{align}

    \begin{proof}
    The sketch of the proof is that, the \textbf{number} of active arms is sum of Binomial random variables, with parameters given the by the policy and transition probabilities. As total number of arms $\rho\to \infty$, each Binomial variable divided by $\rho$ converges to (shifted) Direc Delta. Therefore, the \textit{proportion} of active arms is also (shifted) Direc Delta.
    \paragraph{Notes on Binomial Distribution}
        To make later analysis clear, first consider a single binomial random variable $X$ with $N$ experiments and success rate $p$ (i.e. $X \sim \text{Bin}(N, p)$). For any $x\in [0, 1]$ (assume $Nx$ is integer):
        \begin{align*}
            P[X = Nx]
            & = {N \choose Nx} p^{Nx} (1 - p)^{N(1 - x)}\\
            & \text{Apply Stirling's Formula: $n!\sim \sqrt{2\pi n} \cdot (\frac {n}e)^n$}\\
            & = \sqrt{\frac{1}{x(1 - x)N}} \cdot \left((\frac{p}{x})^{x}(\frac{1 - p}{1 - x})^{(1 - x)}\right)^N
        \end{align*}
        It can be verified that $(\frac{p}{x})^{x}(\frac{1 - p}{1 - x})^{(1 - x)} < 1$ for $x\ne p$. Therefore, as $N\to \infty$
        $$
        P[X = xN] = \begin{cases}
            \sqrt{\frac{1}{x(1 - x)N}} \to \infty& x = p\\
            \sqrt{\frac{1}{x(1 - x)N}} \\\quad \times \mathcal O(\left((\frac{p}{x})^{x}(\frac{1 - p}{1 - x})^{(1 - x)}\right)^N) \to 0& x \ne p
            \end{cases}
        $$
        Therefore, say if we let $f(x) = P[X = xN]$, $f(\cdot)$ is a shifted-to-$p$ Direc Delta function.

        \paragraph{Stationary Distribution \ContextualBudgetBandit{}} Let $A_i^t:= a_i^t \rho$ denote the \textbf{number of active type-$i$ arms} at time point $t$. Conditional on current $A_i^t$ and context $k$,
        \begin{itemize}
            \item $\min(A_i^t, B_{i, k})$ arms are pulled, where each arm remains active w.p. $P_i^ks_i^{t+1} =1 \mid s_i^t=1, a_i^t=1]$.
            \item Each of $\rho - A_i^t$ inactive arms transfers back to active w.p. $P_i^ks_i^{t+1} =1 \mid s_i^t=0]$,
        \end{itemize}   
        Therefore, the number of active arms at next period $A_i^{t + 1}$ is the sum of three binomial random variables:
        \begin{align}
            & \lbrace A_i^{t + 1} \mid A_i^t,  k \rbrace \\
        \sim\ &  \underbrace{\text{Bin}(\min \lbrace A_i^t, B_{i, k}\rbrace, P_i^ks_i^{t+1} =1 \mid s_i^t=1, a_i^t = 1])}_\text{active arms pulled staying active}\\
        & +  \underbrace{\text{Bin}(A_i^t - \min \lbrace A_i^t, B_{i, k}\rbrace, P_i^ks_i^{t+1} =1 \mid s_i^t=1, a_i^t=0])}_\text{idle active arms staying active} \\
        & + \underbrace{\text{Bin}(\rho - A_i^t, P_i^ks_i^{t+1} =1 \mid s_i^t=0])}_\text{inactive arms transfer back to active}.
        \end{align}
        Scaled by $\rho \to \infty$, each of the above binomial distribution converges to a Direc Delta function centered on its mean. Since adding up random variables is equivalent to taking convolution of their probability mass functions---Direc Delta functions are closed under convolution---random variable $a_i^{t + 1} =\frac {A_i^{t + 1}} \rho$'s probability mass function is a Direc Delta shifted by $\frac1\rho\mathbb E[A_i^{t + 1} \mid A_i^t, k]$.
    \end{proof}
\end{lemma}

The lemma implies, the \emph{proportion} of active arms $a_i^t$ evolve ``almost deterministically''---more precisely speaking, fix any policy $\pi$, if at current time step the proportion of active arms is $a_i^t$, context is $k$, the next time step will have $(\mathbb E[a_i^{t + 1} \mid a_i^t, k])$\% active arms almost surely, where $(\mathbb E[a_i^{t + 1} \mid a_i^t, k])$ is given by the following:
\begin{align*}
    \label{eq:E_y_given_x_k}
    & \mathbb{E}[{a_i^{t + 1}} \mid a_i^t, k] \\
    = \ &  \frac{1}{\rho} \mathbb{E}[{A_i^{t + 1}} \mid A_i^t, k] \\
    = \ & \frac{1}{\rho} \mathbb{E}[ \underbrace{\text{Bin}(\min \lbrace A_i^t, B_{i, k}\rbrace, P_i^ks_i^{t+1}=1 \mid s_i^t = 1, a_i^t=1])}_\text{active arms pulled}\\
        & + \underbrace{\text{Bin}(A_i^t - \min \lbrace A_i^t, B_{i, k}\rbrace, P_i^ks_i^{t+1}=1 \mid s_i^t = 1, a_i^t=0])}_\text{untouched active arms} \\
        & + \underbrace{\text{Bin}(\rho - A_i^t, q_i)}_\text{inactive arms}
     ] \\
    = \ & \frac{1}{\rho} (
        \min \lbrace A_i^t, B_{i, k}\rbrace \cdot P_i^ks_i^{t+1}=1 \mid s_i^t = 1, a_i^t=1] \\
        & +
        (A_i^t - \min \lbrace A_i^t, B_{i, k}\rbrace) \cdot P_i^ks_i^{t+1}=1 \mid s_i^t = 1, a_i^t=0] \\
        & +
        (\rho - A_i^t)q_i
    ) \\
    \end{align*}
    Denote $\beta_{i, k} := \frac{B_{i, k}}{\rho}$:
    \begin{align}
    & \mathbb{E}[{a_i^{t + 1}} \mid a_i^t, k]\\
    \label{eq:mu_k(A)_expression_1}
    = \ &\min(a_i^t, \beta_{i, k}) \cdot P_i^ks_i^{t+1}=1 \mid s_i^t = 1, a_i^t=1] \\
    \label{eq:mu_k(A)_expression_2}
     & + \max(a_i^t - \beta_{i, k}, 0) \cdot P_i^ks_i^{t+1}=1 \mid s_i^t = 1, a_i^t=0]\\
     \label{eq:mu_k(A)_expression_3}
    & + (1 - a_i^t)P_i^ks_i^{t+1}=1 \mid s_i^t = 0]
\end{align}

 If, current time step's proportion of active arms is $x\in [0, 1]$, with probability $f_k$ context $k$ occurs, then the next time step's active-arm proportion will be $y = \mathbb E[{a_i^{t + 1}} \mid x, k]$ (as given in~\ref{eq:mu_k(A)_expression_1}-\ref{eq:mu_k(A)_expression_3}) w.p. $f_k$. And for each $y$, define its inverse
\begin{align}
\label{eq:mathcal_X_y}
    \mathcal X (y) :=\lbrace (x, k): \mathbb E[a_i^{t + 1}\mid x, k] = y \rbrace.
\end{align}
Denote the stationary distribution of proportion of active arms as $\pi:[0, 1]\to [0, 1]$, it should satisfy:
\begin{align}
    \label{eq:stationary_distribution_expression}
    \pi(y) = \sum_{(x, k)\in \mathcal X(y)} f_k\pi(x).
\end{align}

\subsection{A 5/6 Approximation Upperbound.}
\label{appendix2_5/6_instnace}
\paragraph{An adversarial instance}
Consider a base \CBB{} example with only one type of arm (i.e., $M=1$).  Let there be $\rho$ copies of this arm in the scaled setting as $\rho \to \infty$.  We drop the index $i$ for convenience.  Suppose there are two contexts, $k \in \{1, 2\}$, each occurring with probability $f_1 = f_2 = 0.5$.  

Let $\epsilon>0$.  The transition probabilities and rewards are defined as follows.
\begin{itemize}
    \item Context 1: transition probabilities is
    \begin{align*}
       &  P^1[s^{t+1}=1\mid s=1, a=1]=1-\epsilon\\
       &  P^1[s^{t+1}=0\mid s=1, a=1]=\epsilon\\
       &  P^1[s^{t+1}=1\mid s=1, a=0]=1\\
       & P^1[s^{t+1}=0\mid s=1, a=0]=0\\
       & P^1[s^{t+1} =1\mid s=1, \forall a= 0, 1]=1\\
       & P^1[s^{t+1} =0\mid s=1, \forall a= 0, 1]=0\\
    \end{align*}
    reward for context 1:
    \begin{align*}
        & r(s^t=1, a^t=1; k=1) = 1\\
        & r(s^t=1, a^t=0; k=1) = 0\\
        & r(s^t=0, a^t=1; k=1) = 0\\
        & r(s^t=1, a^t=0; k=1) = 0
    \end{align*}
    \item Context 2: transition probabilities is
    \begin{align*}
       &  P^2[s^{t+1}=1\mid s=1, a=1]=0\\
       &  P^2[s^{t+1}=0\mid s=1, a=1]=1\\
       &  P^2[s^{t+1}=1\mid s=1, a=0]=1\\
       &  P^2[s^{t+1}=0]\mid s=1, a=0]=0\\
       &  P^2[s^{t+1} =1\mid s=1, \forall a= 0, 1]=1\\
       &  P^2[s^{t+1} =0\mid s=1, \forall a= 0, 1]=0\\
    \end{align*}
    reward for context 2:
    \begin{align*}
        & r(s^t=1, a^t=1; k=1) = 1 + \epsilon\\
        & r(s^t=1, a^t=0; k=1) = 0\\
        & r(s^t=0, a^t=1; k=1) = 0\\
        & r(s^t=1, a^t=0; k=1) = 0
    \end{align*}
\end{itemize}

\paragraph{Bugdet} Assume that budget is $1/3$ of the number of total arms. I.e. in the $\rho$-scaled instance, $B=\lfloor \frac 13\rfloor$. As the scaling factor $\rho \to \infty$, we can without loss of generality assumes that it's an interger.

\paragraph{The Reward for \ContextualOccupancyIndexPolicy{}}

The \OMLP{} for the base instance simplifies to
\begin{align*}
    \max_{\mu, B_k} \quad & \mu(1, 1, 1)  + \mu(1, 1, 2)(1 + \epsilon)\\
    \text{subject to}\quad & \\
    & (1 - P[s = 1]) = \epsilon\mu(1, 1, 1) + \mu(1, 1, 2)\\
    & \mu(1, 1, k) \le \frac12 P[s = 1], \forall k = 1, 2\\
    & \mu(1, 1, 1) + \mu(1, 1, 2) \le \frac13
\end{align*}
The \ContextualOccupancyIndexPolicy{} then allocate budget following the optimal solution ($\mu^\star$) of the \OMLP{}. For the $\rho$-scaled \CBB{} with total budget $B=\frac13 \rho$, the budget allocation of \ContextualOccupancyIndexPolicy{} is
\begin{align*}
    B_1 & = \rho \times \frac1{f_1}\mu^\star(1, 1, 1) = 0,\\
    B_2 & = \rho \times \frac1{f_2} \mu^\star(1, 1, 2) = \frac23.
\end{align*}
From~(\ref{eq:mu_k(A)_expression_1}-\ref{eq:mu_k(A)_expression_3}) we obtain, as $\rho\to\infty$, the transition dynamic of the proportion of active arms $x^t\to x^{t+1}$ in RMAB:
\begin{itemize}
    \item With probability $f_1=0.5$, context~$k=1$:
    $$
    x^{t+1} = \mathbb E[a_i^{t + 1}\mid x^t, k]  = 1;
    $$
    \item With probability $f_2=0.5$, context~$k=2$:
    $$
    x^{t+1} =\mathbb E[a_i^{t + 1}\mid x^t, k]  = \max(x^t - \frac23, 0) + 1 - x^t.
    $$
\end{itemize}

From~\ref{eq:mathcal_X_y} and~\ref{eq:stationary_distribution_expression} we obtain the stationary distribution $\pi$ under $\policy^\star$: (actually, guess-and-verify)
\begin{align*}
    &\pi(\frac13) = \frac13,\\
    &\pi(\frac23) = \frac16\\
    &\pi(1) = \frac12,\\
    &\pi(x) = 0, \text{otw.}
\end{align*}
When the proportion of active arms$=\frac13$---only half of the budget is utilized. This happens, as give above, w.p. $\pi(\frac13) = \frac13$. So the reward as $\rho\to\infty$ is
\begin{align*}
    & \Reward^\text{\ContextualOccupancyIndexPolicy{}}(\rho) \\
    =\ &  f_2( \frac13 \rho\pi(\frac13) + \frac23 \rho (\pi(\frac23) + \pi(1)) \\
    =\ & 0.5\rho(\frac19 + \frac49) = \frac5{18}\rho
\end{align*}

\paragraph{Optimal Budget Allocation}
However, notice that the other context $k = 1$ is almost always active (it has probability $p = \epsilon$ of transfer to inactive). Therefore, if we allocate all budget to context~$1$:
$$B_1=\frac23, B_2=0$$.
The stationary reward for the optimal budget allocation is
$$
\Reward^\text{ContextOpt} (\rho)= \frac13 \rho
$$

Therefore, the \ContextualOccupancyIndexPolicy{}'s approximation is bounded above by $\frac 56$.
$$
  \lim_{\rho\to\infty}\frac{\Reward^\text{\ContextualOccupancyIndexPolicy{}}(\rho)}{\Reward^\text{ContextOpt}(\rho)} = \frac56.
$$
\end{proof}

\subsection{Remark: Closed-form unavailable}

Ending remark for this Appendix section, and as a complement to the asymptotic analysis of \CBB{}, we provided the following example, where, the closed-form solution of the staionary distribution of the proportion of active arms can only be calculated numerically but not characterized in clean closed-form as the above example.

\paragraph{Single-Type Base Example 2}

Consider a base \CBB{} example with only one type of arm (i.e., $M=1$).  Let there be $\rho$ copies of this arm in the scaled setting as $\rho \to \infty$.  We drop the index $i$ for convenience.  Suppose there are two contexts, $k \in \{1, 2\}$, each occurring with probability $f_1 = f_2 = 0.5$. 
\textbf{Transition Probabilities and Rewards.} Let $\epsilon>0$.  The transition probabilities and rewards are defined as follows.

\paragraph{Transition Probabilities and Rewards.}
Let $\epsilon>0$.  The transition probabilities and rewards are defined as follows.
\begin{itemize}
    \item \textbf{Context 1:} 
    \emph{Transition probabilities}:
    \begin{align*}
       &P^1[s^{t+1} = 1 \mid s = 1, a = 1] = 1 - \epsilon,\\
       &P^1[s^{t+1} = 0 \mid s = 1, a = 1] = \epsilon,\\
       &P^1[s^{t+1} = 1 \mid s = 1, a = 0] = 1,\\
       &P^1[s^{t+1} = 0 \mid s = 1, a = 0] = 0,\\
       &P^1[s^{t+1} = 1 \mid s = 0, \forall a=0, 1] = \frac12\\
       &P^1[s^{t+1} = 0 \mid s = 0, \forall a=0, 1] = \frac12\\
    \end{align*}
    \emph{Rewards}:
    \begin{align*}
       &r(s^t = 1, a^t = 1; k=1) = 1,\\
       &r(s^t = 1, a^t = 0; k=1) = 0,\\
       &r(s^t = 0, a^t = 1; k=1) = 0,\\
       &r(s^t = 0, a^t = 0; k=1) = 0.
    \end{align*}
    
    \item \textbf{Context 2:}
    \emph{Transition probabilities}:
    \begin{align*}
       &P^2[s^{t+1} = 1 \mid s = 1, a = 1] = 0,\\
       &P^2[s^{t+1} = 0 \mid s = 1, a = 1] = 1,\\
       &P^2[s^{t+1} = 1 \mid s = 1, a = 0] = 1,\\
       &P^2[s^{t+1} = 0 \mid s = 1, a = 0] = 0,\\
       &P^2[s^{t+1} = 1 \mid s = 0, \forall a=0, 1] = \frac12\\
       &P^2[s^{t+1} = 0 \mid s = 0, \forall a=0, 1] = \frac12\\
    \end{align*}
    \emph{Rewards}:
    \begin{align*}
       &r(s^t = 1, a^t = 1; k=2) = 1 + \epsilon,\\
       &r(s^t = 1, a^t = 0; k=2) = 0,\\
       &r(s^t = 0, a^t = 1; k=2) = 0,\\
       &r(s^t = 0, a^t = 0; k=2) = 0.
    \end{align*}
\end{itemize}
\paragraph{Budget Constraint.}
Assume the budget in each round is a fraction of the total number of arms.  For concreteness, let the budget be 
\[
  B \;=\; \Big\lfloor \tfrac{1}{4}\,\rho \Big\rfloor,
\]
so that we may activate at most $\lfloor \tfrac{\rho}{4} \rfloor$ arms (out of $\rho$).  As $\rho \to \infty$, we can assume without loss of generality that $B = \tfrac{\rho}{3}$ is an integer.

The \OMLP{} simplifies to
\begin{align*}
    \max_{\mu,\,B_k} \quad 
    & \mu(1,1,1) \;+\; \mu(1,1,2)\,\bigl(1 + \epsilon\bigr) \\
    \text{subject to}\quad 
    & \tfrac12\bigl(1 - P[s=1]\bigr) \;=\; \epsilon\,\mu(1,1,1) \;+\; \mu(1,1,2),\\
    & \mu(1,1,k) \;\le\; \tfrac12\,P[s=1],\quad \forall\,k \in \{1,2\},\\
    & \mu(1,1,1) + \mu(1,1,2) \;\le\; 0.25 
\end{align*}

The optimal solution is $P^\star[s = 1] = 0.5, \mu^\star(1, 1, 1) = 0, \mu^\star(1, 1, 2) = 0.25$. Similarly, \ContextualOccupancyIndexPolicy{} would allocate budget so that
\begin{align*}
    B_1 & = 0,\\
    B_2 & = 0.5\rho.
\end{align*}

Therefore, from~\ref{eq:mathcal_X_y} and~\ref{eq:stationary_distribution_expression} we obtain for the stationary distribution $\pi$:
\begin{align}
    \pi(y) = \begin{cases}
        0 & y\in (0, \frac14)\\
        \frac12 \pi(\frac12) & y = \frac14 \\
        \frac12 \pi(1 - 2y) + \frac12 \pi(2y) & y\in (\frac14, \frac12]\\
        \frac12\pi(2y  - 1) & y\in (\frac12, 1].
    \end{cases}
\end{align}
It doesn't have a clean closed-form solution. But the stationary of proportion of active arms ($\pi(\cdot)$ can be solved numerically, as shown in Figure~\ref{fig:stationary_distribution}.
\begin{figure}
    \centering
    \includegraphics[width=1\linewidth]{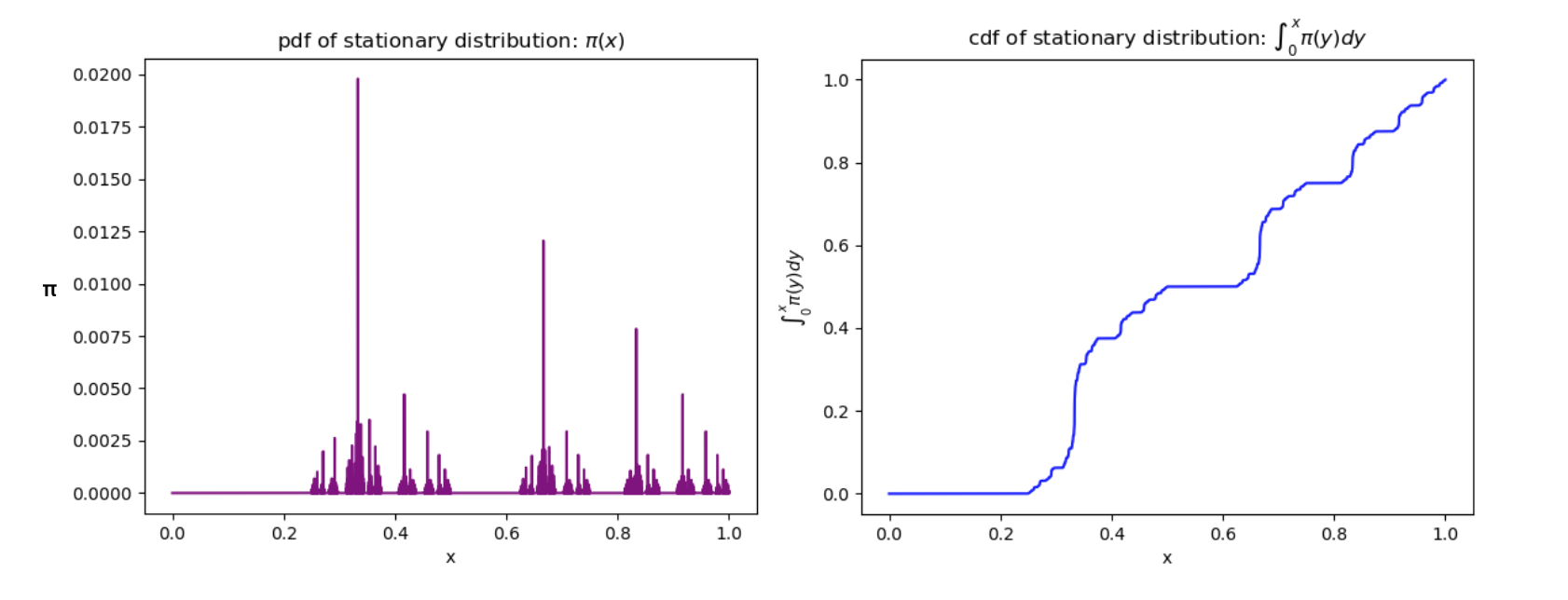}
    \caption{Calculated stationary distribution.}
    \label{fig:stationary_distribution}
\end{figure}
As shown in Figure~\ref{fig:stationary_distribution}, for nontrivial probability, the proportion of active arms is less than $0.5$---less than the required active arms to pull. The stationary reward can be calculated as
\begin{align*}
    & \Reward^\text{\ContextualOccupancyIndexPolicy{}}\\
    = \ & \int_0^1 \sum_k f_k r(k) \rho \min(\beta_k, x) \pi(x)\,dx\\
    = \  & \rho\int_0^1\frac12 (1 + \epsilon) \min(\frac12, x) \pi(x) \,dx \\
    \approx\ &  \rho (1 + \epsilon) 0.214.
\end{align*}
However, notice that the other context $k = 1$ is almost always acive (it has probability $p = \epsilon$ of transfer to inactive). Therefore, if we allocate all budget to it---almost all arms will be active all the time, and reward of $r(1) = 1$ can be accured at every pull. By back-on-the-envelope calculation, under this budget allocation (all to context $1$) the system generate (almost) exactly $\frac14$ reward. Therefore, this instance give an lowerbound of $0.214/0.25  = 0.856$ impossibility lowerbound for the LP-induced budgets.
\section{Experiment Details: Design and Implementation}

\subsection{Low/High Activeness in Synthetic Food Rescue~\CBB{}}
\label{appendix:activeness_details}
We blend two types of synthetic setups to merge and simulate different dynamics in formulating the food rescue \CBB{}:
\subsubsection{Organic}
In the organic instance, $N$ volunteers and $K$ regions are randomly positioned on a two-dimensional plane. Each volunteer and region is associated with a location and attributes—namely, volunteer activeness, region popularity, and a historical record $H_i$ (which, in turn, influences the context probabilities $f_k$). For every volunteer $i$ and region $k$, we define the pick-up rate as
$$
p_{i,k} = \exp\Biggl( \alpha\, \mathrm{pop}_k - \gamma\, d(i,k) + \beta\, \frac{|H_i|}{H_{\max}} \Biggr),
$$

where

- $\alpha$ is the parameter capturing the influence of region popularity (with $\mathrm{pop}_k$ denoting the popularity of region $k$),
- $\gamma$ is the distance sensitivity parameter (with $d(i,k)$ representing the distance between volunteer $i$ and region $k$),
- $\beta$ is the parameter reflecting volunteer activeness (with $|H_i|$ being the size of volunteer $i$'s history), and
- $H_{\max}$ is a normalization constant.

Transition dynamics are such that an active volunteer (state $s=1$) who is notified (action $a=1$) picks up the task with probability $p_{ik}$ and may then become inactive. The immediate reward for a notification is a function of region popularity and $p_{ik}$.

\subsubsection{Churner}

In addition to the organic instance, we define a churner instance to capture more challenging dynamics within the \CBB{} framework. 

The $K$ regions are partitioned into \emph{preferred} regions ($\mathcal{K}_\text{preferred} \subsetneq [K]$ and its complement. The preferred regions are designed such that, for any volunteer $i$, the pick-up probabilities $p_{ik}$ for $k \in \mathcal{K}_\text{preferred}$ are drawn uniformly around a high mean (e.g., centered at 0.95), making these regions very attractive and yielding a high probability of transitioning a volunteer to an inactive state. In contrast, for regions $k\notin \mathcal{K}_\text{preferred}$ (the ``disliked'' regions), the transition probabilities are concentrated around a low mean (e.g., centered at 0.05). Recovery probabilities $q_i$ for volunteers are generated around a prescribed mean (e.g., 0.2). Moreover, the reward structure is modified so that notifications in preferred regions yield an elevated immediate reward to reflect their allure despite the adverse long-term effect. 

\subsubsection{Blended Instance}

Finally, we construct a \emph{Blended Instance} that merges organic and churner dynamics. A fraction, referred to as abundance, $\rho_\text{Abundance} \in [0,1]$ of the $N$ volunteers is designated to follow churner dynamics, while the remaining $N - N_{\text{active}}$ volunteers follow organic dynamics. Formally, we set
$$
N_{\text{organic}} = \lfloor \rho_\text{Abundance} N \rfloor,
$$
and generate two independent instances over the same set of $K$ regions:
\begin{enumerate}[label=(\roman*)]
    \item An \emph{organic instance} with $N_{\text{organic}}$ volunteers. The transition dynamics and rewards are constructed as described in Section~\ref{appendix:activeness_details}
    \item A \emph{churner instance} with $N - N_{\text{organic}}$ volunteers, constructed as described in Section~\ref{appendix:activeness_details}.
\end{enumerate}

\subsection{Ablation Experiments on Synthetic Food Rescue Instance}
\label{appendix:experiment_details}

\begin{figure}
    \centering
    \includegraphics[width=1\linewidth,page=6]{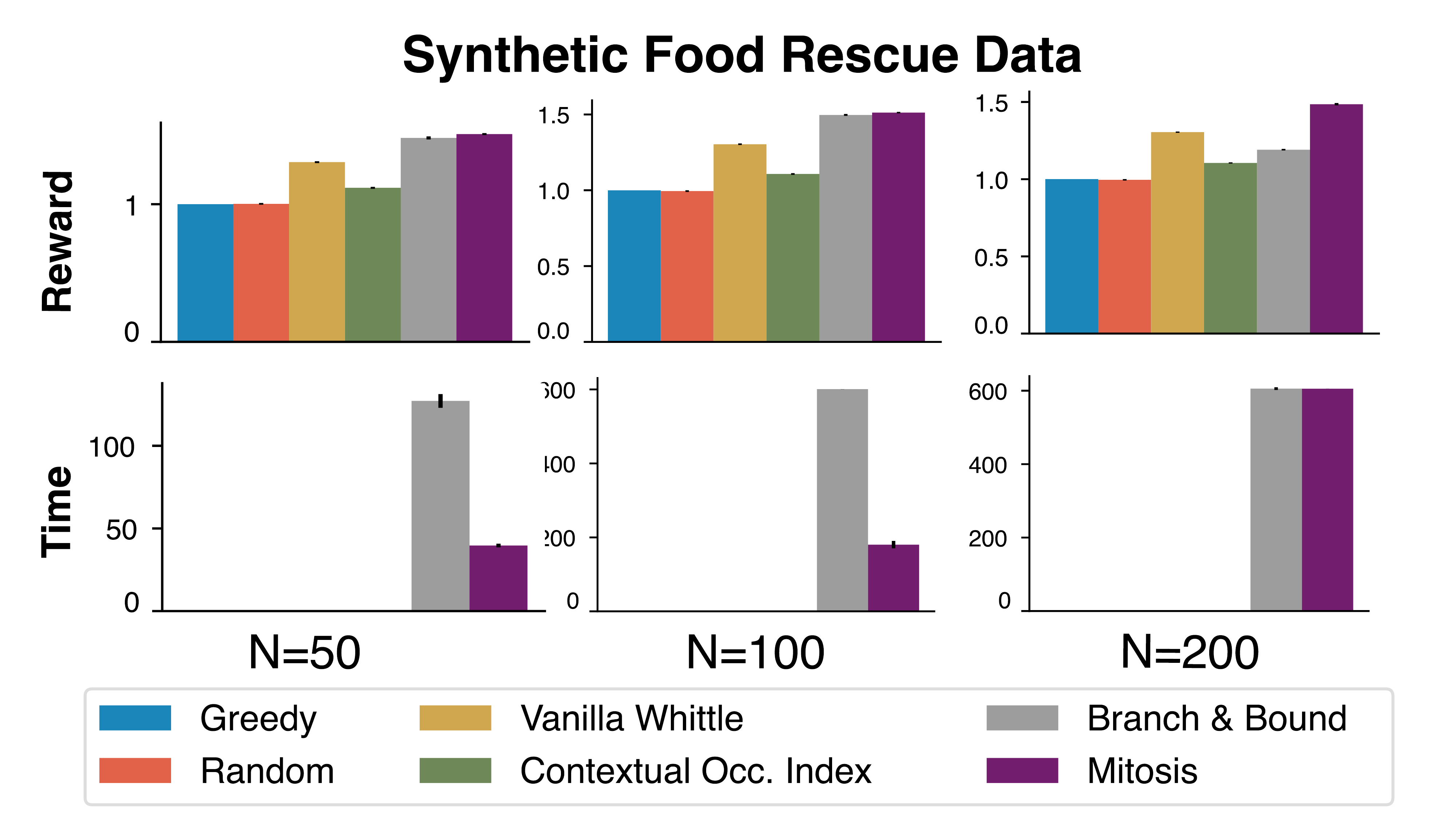}
    \caption{Ablation Experiments on Synthetic Food Rescue Experiments, Varying Number of Volunteers}
    \label{fig:ablation_N_synthetic}
\end{figure}

\begin{figure}
    \centering
    \includegraphics[width=1\linewidth,page=1]{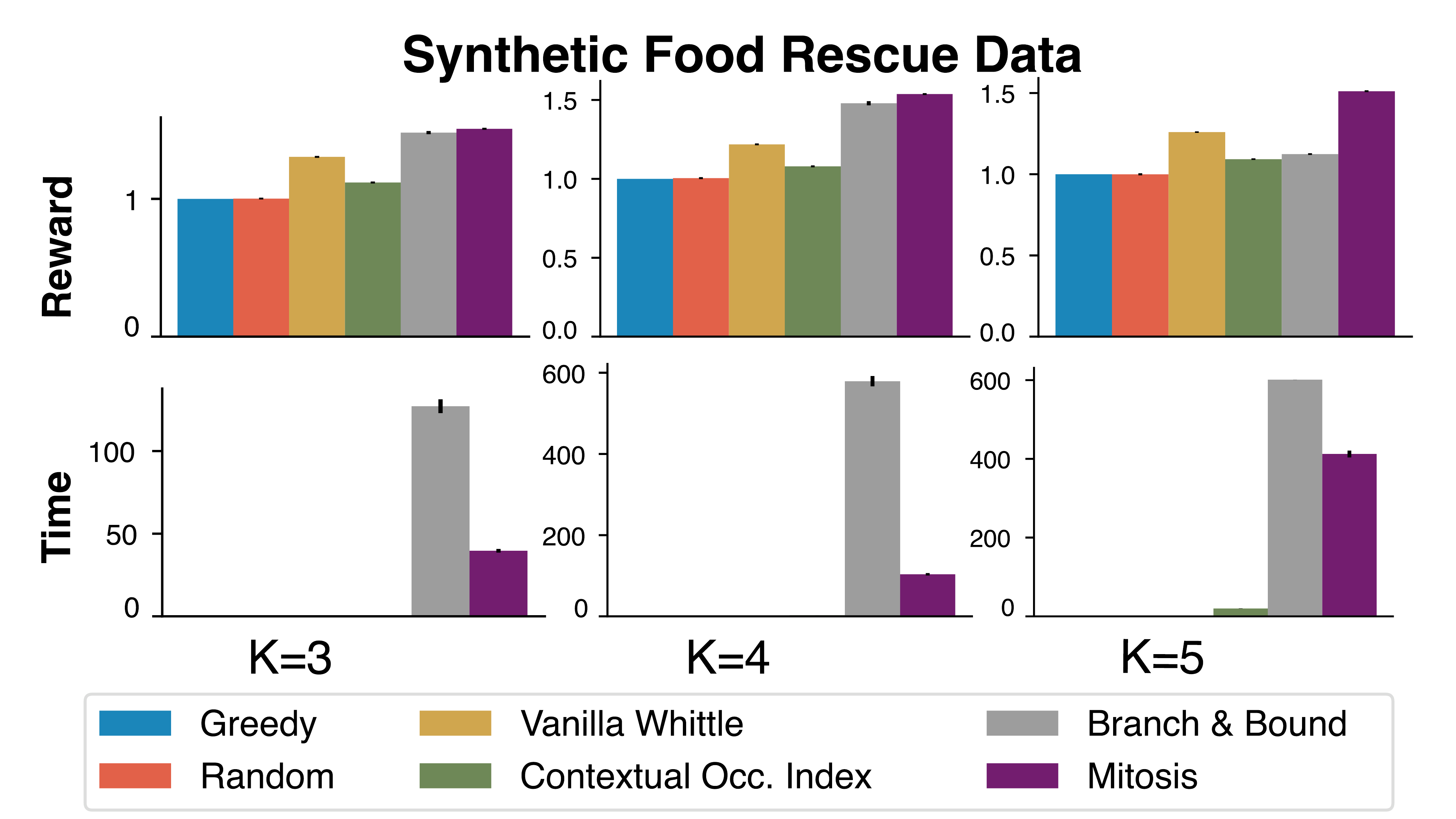}
    \caption{Ablation Experiments on Synthetic Food Rescue Experiments, Varying Number of Contexts}
    \label{fig:ablation_K_synthetic}
\end{figure}

\begin{figure}
    \centering
    \includegraphics[width=1\linewidth,page=3]{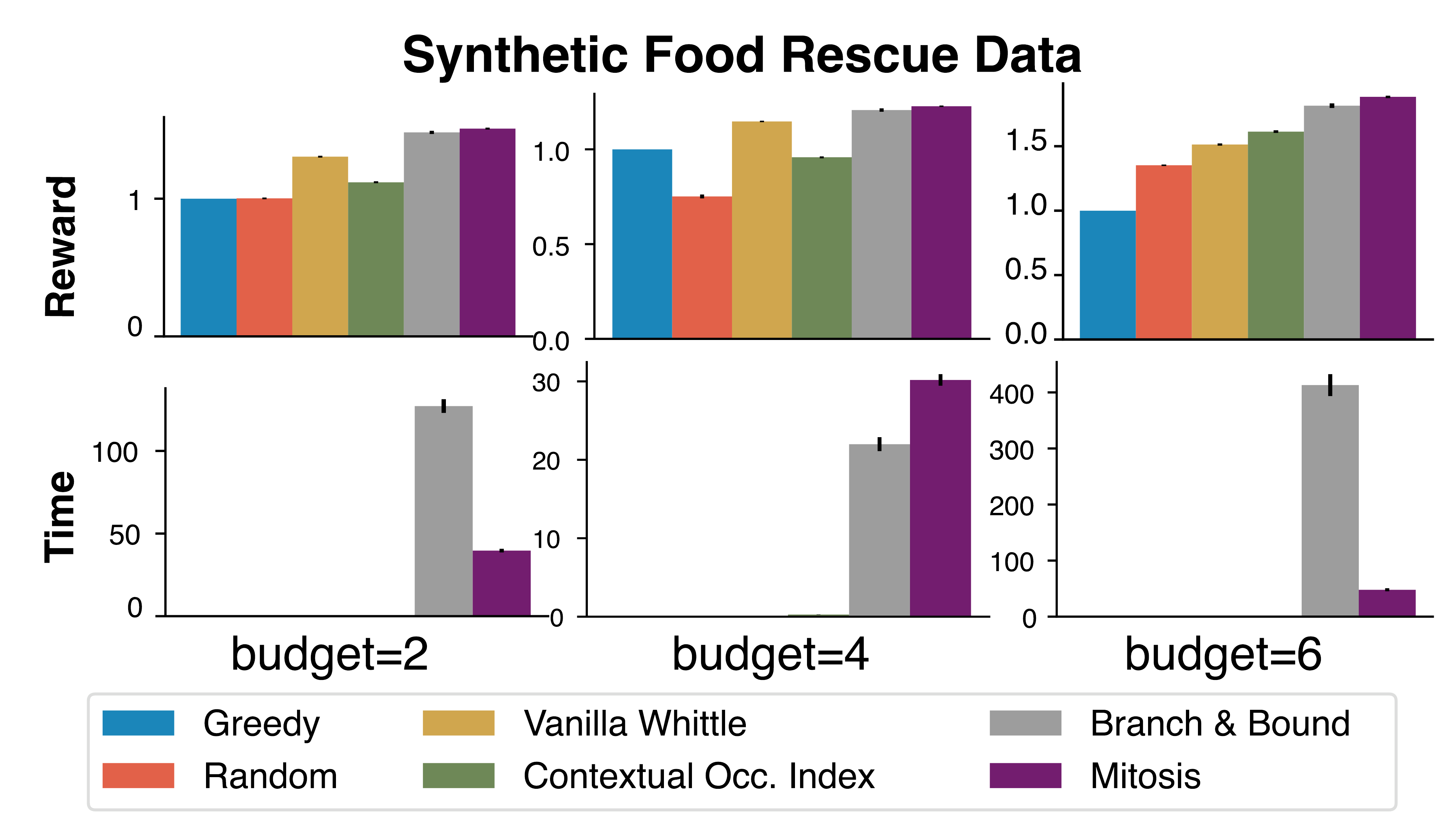}
    \caption{Ablation Experiments on Synthetic Food Rescue Experiments, Varying Budget}
    \label{fig:ablation_B_synthetic}
\end{figure}

\begin{figure}
    \centering
    \includegraphics[width=1\linewidth,page=5]{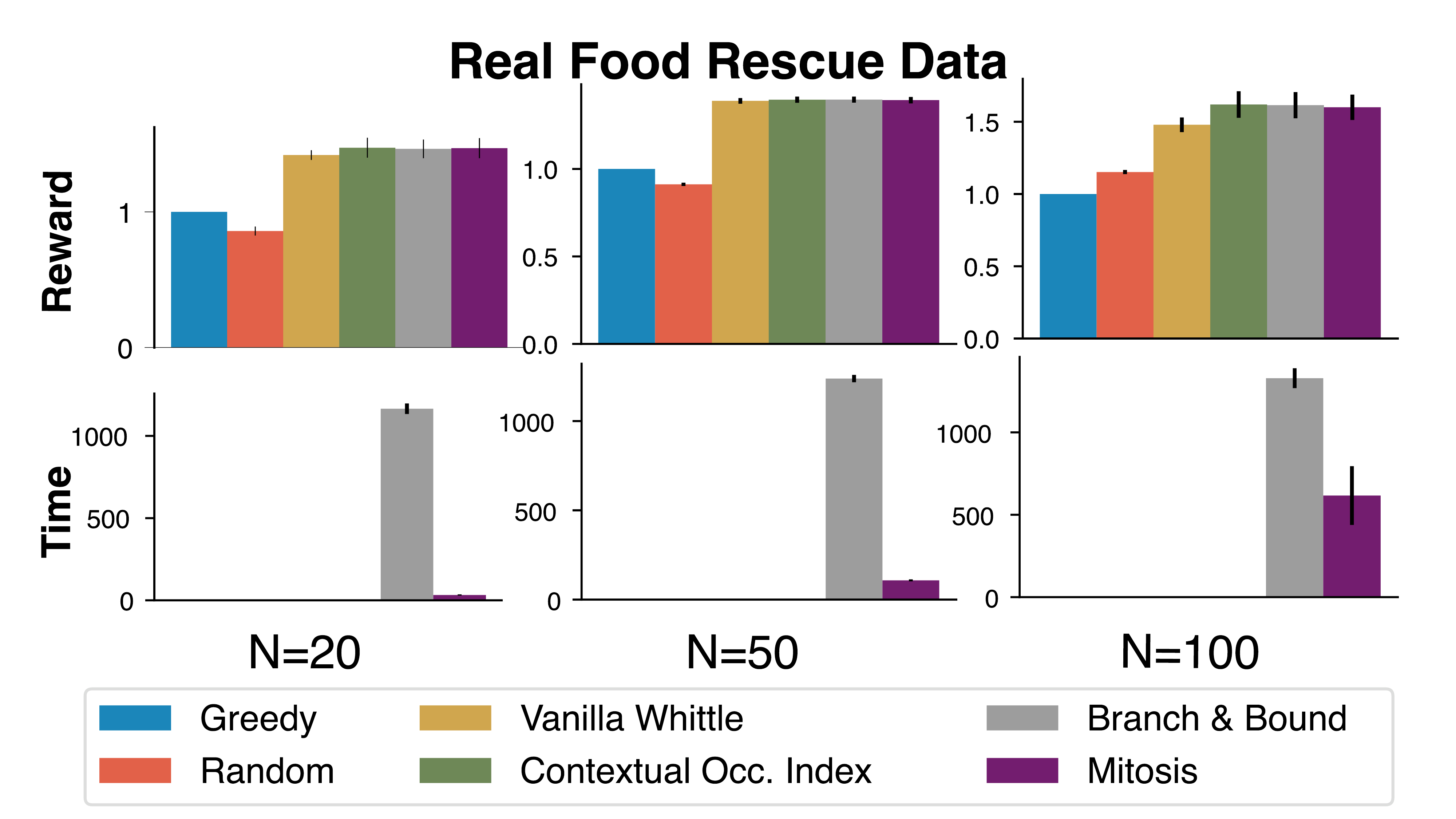}
    \caption{Ablation Experiments on Real Food Rescue Experiments, Varying Number of Volunteers}
    \label{fig:ablation_N_real}
\end{figure}

\begin{figure}
    \centering
    \includegraphics[width=1\linewidth,page=4]{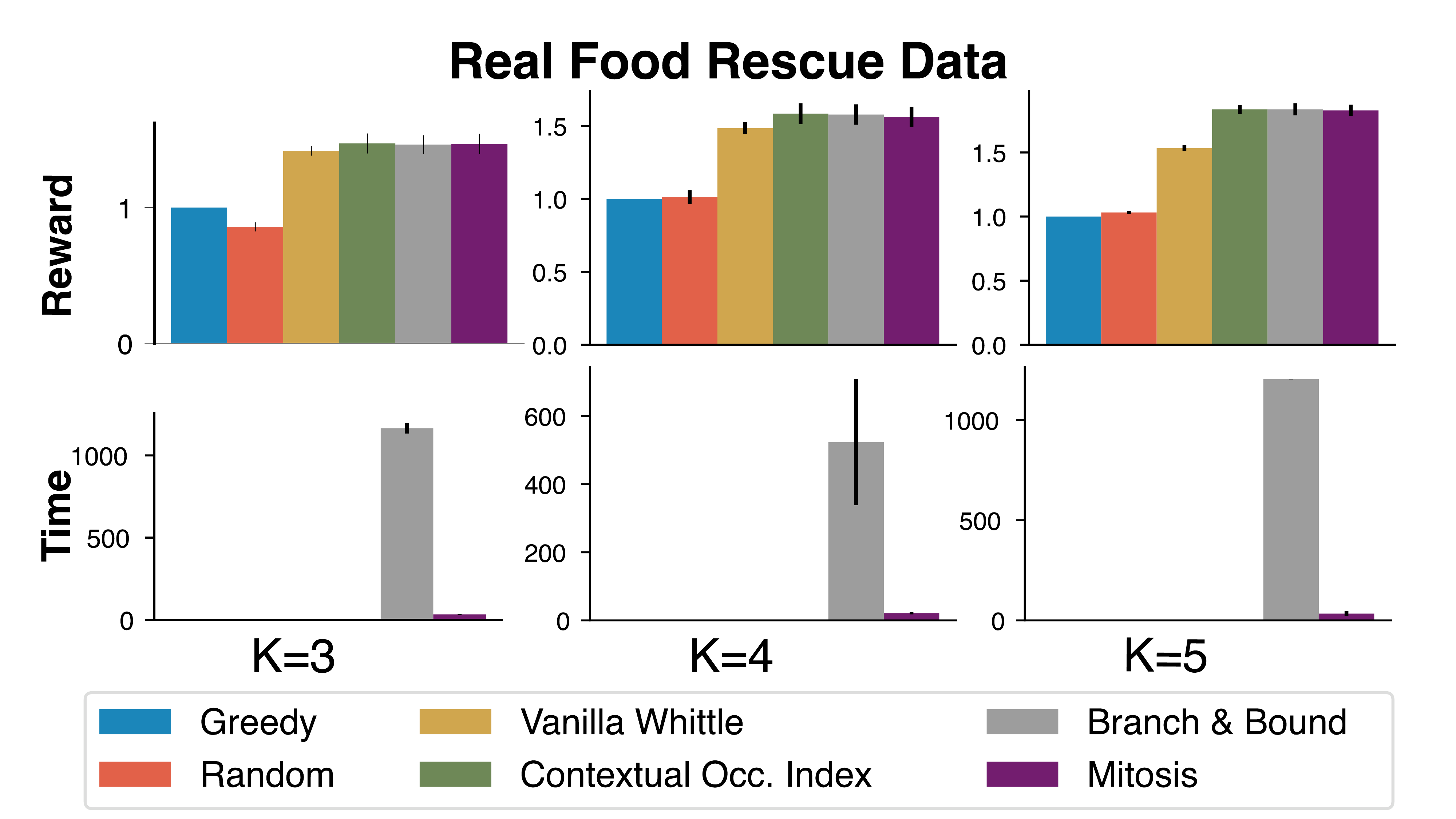}
    \caption{Ablation Experiments on Real Food Rescue Experiments, Varying Number of Contexts}
    \label{fig:ablation_K_real}
\end{figure}

\begin{figure}
    \centering
    \includegraphics[width=1\linewidth,page=2]{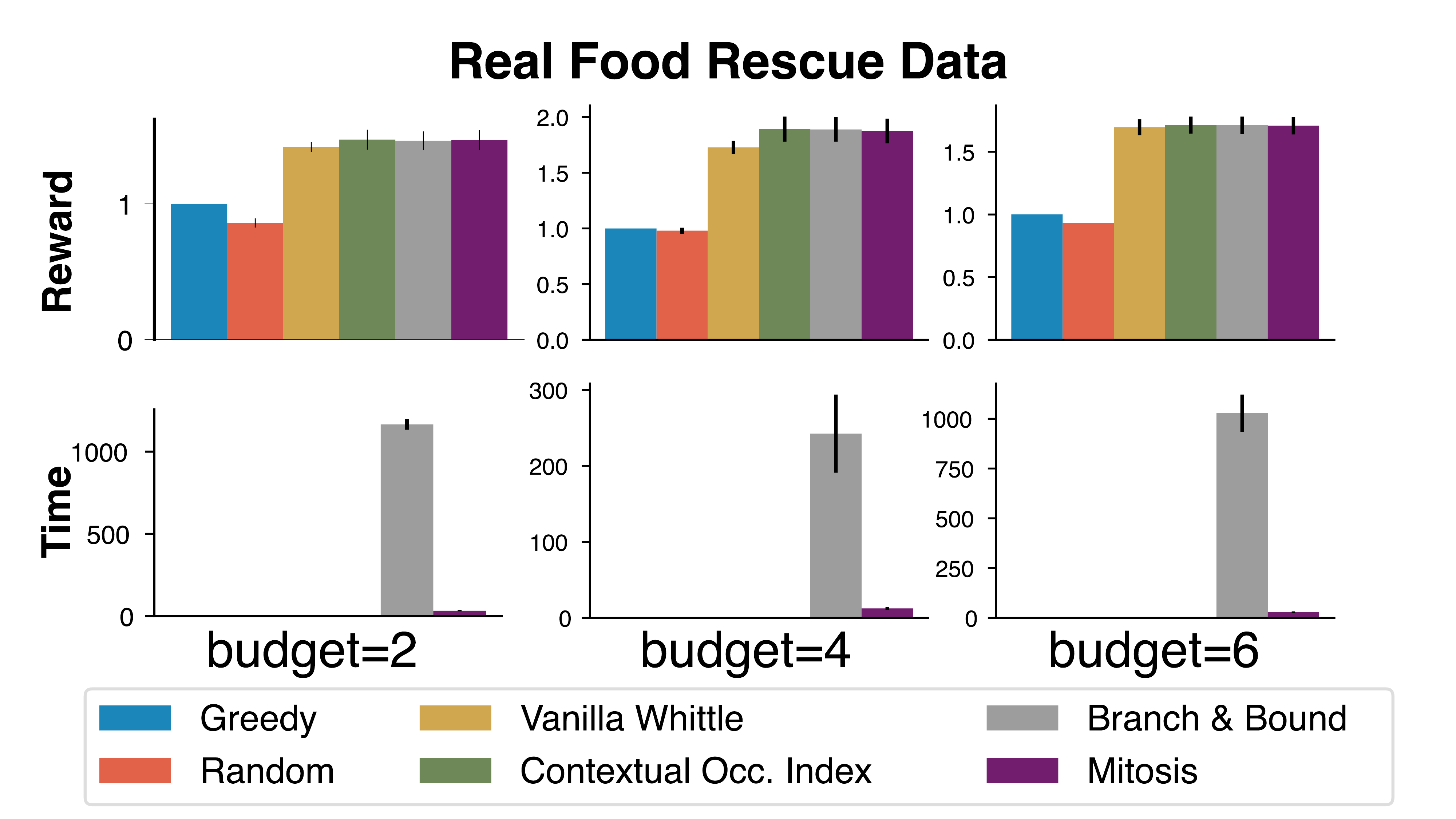}
    \caption{Ablation Experiments on Synthetic Food Rescue Experiments, Varying Budget}
    \label{fig:ablation_B_real}
\end{figure}

Below, we summarize the ablation study results on synthetic data, where we systematically vary the number of volunteers $(N)$, the number of regions $(K)$, and the budget $(B)$.

\begin{itemize}
\item \textbf{Varying number of volunteers for $N=50, 100, 200$, fix $K=3$ regions and budget be $5\%$ number of volunteers (Figure~\ref{fig:ablation_N_synthetic}):}  
  As $N$ increases, \ZoomingBranch{} (purple) consistently leads in reward and remains much faster than \BranchAndBound{} (gray). Note that in~$N=200$ both \BranchAndBound{} and \ZoomingBranch{} reach the time limit (600s) and is terminated, but sill within the same time limit, the \ZoomingBranch{}'s solution is more than that of \BranchAndBound{}'s, demonstrating that \ZoomingBranch{} is much faster. \ContextualOccupancyIndexPolicy{} (green) still lags in reward, indicating it does not fully exploit the increased volunteer pool.

\item \textbf{Varying number of regions for $K=3, 4, 5$, fix number of volunteers $N=50$, budget $B=5$ (Figure~\ref{fig:ablation_K_synthetic}):}  
  With more regions, the search space for budget increases exponentially. \BranchAndBound{} maintains a slight reward edge but at a steep runtime cost, it times out already at $k=4$. \ZoomingBranch{} remains best and is faster compared to \BranchAndBound{}. \ContextualOccupancyIndexPolicy{} gains some benefit but continues to underperform compared to the optimal. 

\item \textbf{Varying budget $B=2, 4, 6$, fix volunteers~$N=50$, regions~$K=3$ (Figure~\ref{fig:ablation_B_synthetic}):}  
  Increasing $B$ allows more notifications, boosting \ZoomingBranch{} substantially while also helping \ContextualOccupancyIndexPolicy{} close some of the gap. Once again, \BranchAndBound{} yields top‐tier rewards but incurs much higher computation time.
\end{itemize}

\subsection{Ablation Experiments on Real Food Rescue Data} Similar as synthetic data's ablations, we systematically vary the number of volunteers $(N)$, the number of regions $(K)$, and the budget $(B)$ of \CBB{} constructed on real food rescue data. Results are shown in 
\begin{itemize}
    \item Figure~\ref{fig:ablation_N_real}: changing $N=20, 50, 100$ while maintain number of regions $K=3$, budget $B=2$.
    \item Figure~\ref{fig:ablation_K_real}: changing $K=3, 4, 5$ while maintaining number of regions $N=20$, budget $B=2$.
    \item Figure~\ref{fig:ablation_B_real}: changing $B=2, 4, 6$ while maintaining number of volunteers $N=20$, number of regions $K=3$.
\end{itemize}
As the scale of the instance increases ($N, K$ or $B$ increases), \VanillaWhittle{} performance grows worse compared to \ContextualOccupancyIndexPolicy{}, \BranchAndBound{}, \ZoomingBranch{} which they perform similarly. This shows that (i) when the scale of the problem increase, it is necessary to introduce context-aware policies to reach optimal performance (ii) in application, \ContextualOccupancyIndexPolicy{} is sufficient for near-optimal performance. \ZoomingBranch{} guarantees optimality and is significantly faster than \BranchAndBound{}.

\section{No-Regret Guarantee for the Mitosis Algorithm}
\label{appendix:mitosis}

In the MAB framework each arm represents a candidate budget allocation \(\vec{B}\) from the feasible set
\[
\mathcal{B}_0 \triangleq \Bigl\{ \vec{B} \in \mathbb{N}^K : \sum_{k=1}^K B_k \le B \Bigr\}.
\]
Pulling an arm \(\vec{B}\) corresponds to calling the fast oracle \(\OracleSmall(\vec{B})\) (with \(\texttt{epoch}=1\)) which returns a noisy estimate of the reward \(\mu(\vec{B})\). Thus, by running a Multi-Armed Bandit (MAB) algorithm over the arms \(\vec{B} \in \mathcal{B}_0\) we aim to select the arm with the highest expected reward without having to estimate \(\mu(\vec{B})\) for every \(\vec{B}\).

\begin{definition*}[Reward Regret]
Let
\[
\mu^\star \triangleq \max_{\vec{B}\in\mathcal{B}_0}\mu(\vec{B})
\]
and denote by \(\vec{B}_t\) the budget allocation (arm) chosen at time \(t\). Then the instantaneous regret at time \(t\) is
\[
\Delta_t \triangleq \mu^\star - \mu(\vec{B}_t),
\]
and the cumulative (reward) regret over a time horizon \(T\) is defined as
\[
R(T) \triangleq \sum_{t=1}^T \Delta_t = \sum_{t=1}^T \Bigl(\mu^\star - \mu(\vec{B}_t)\Bigr).
\]
\end{definition*}

The goal is to design an algorithm whose cumulative regret grows sublinearly in \(T\); that is, \(\frac{R(T)}{T}\to 0\) as \(T\to\infty\). In our setting, the optimal budget allocation \(\vec{B}^\star\) (with \(\mu(\vec{B}^\star)=\mu^\star\)) will be identified as \(T\) increases.

\thmMitosisNoRegret*

\begin{proof}
The proof is built on the classical UCB1 analysis of \cite{UCB1_2022}. In the \ZoomingBranch{} Algorithm (Algorithm~\ref{alg:Mitosis}), each arm $\vec{B}$ is initialized with its upperbound $\LP(\vec{B})$. The algorithm maintains two types of arms:
\begin{itemize}
    \item \textbf{Unpromising arms:} Arms that are encapsulated in \TreeArm{}, who has not yet been pulled. Their index is given by $\LP(\vec{B})$.
    \item \textbf{Candidate arms:} Arms that have been pulled at least once. For these, the UCB index at time $t$ is defined as
        $$
            I_t(\vec{B}) = \hat{\mu}_t(\vec{B}) + c\sqrt{\frac{\log t}{N_t(\vec{B})}},
        $$
    where $\hat{\mu}_t(\vec{B})$ is the empirical mean, $N_t(\vec{B})$ is the number of pulls, and $c>0$ is a constant.
\end{itemize}

The algorithm runs for $T$ rounds; by the end, let $\mathcal{A}$ denote the final set of candidate arms. We consider two cases based on the location of the optimal arm $\vec{B}^\star$.
\medskip

\noindent\emph{Case 1:} $\vec{B}^\star \in \mathcal{A}$.  
In this case, the optimal arm has been pulled at least once. Therefore, the candidate arms $\mathcal A$ form a sub-MAB instance where we can directly apply UCB1's regret bound on; the arms in \TreeArm{} are not pulled anyway, so they do not contribute to regret. Standard UCB1 analysis (using a peeling argument and concentration inequalities, see~\cite{UCB1_2022}) shows that the expected pulls on any suboptimal arms, denoted by $N_t(\vec B)$, satisfy
\begin{equation}
    \label{eq:expected_pulls_UCB}
    \mathbb{E}[N_t(\vec{B})] = \mathcal O\Biggl(\frac{\log t}{\Delta(\vec{B})^2}\Biggr).
\end{equation}
Thus, the regret incurred by arms in $\mathcal{A}$ is
$$
R(T) = \sum_{\vec{B}\in \mathcal{A}} \mathbb{E}[N_T(\vec{B})]\,\Delta(\vec{B}) = \mathcal O\Biggl( \sum_{\vec{B}\in \mathcal{A}} \frac{\log T}{\Delta(\vec{B})} \Biggr).
$$
\medskip

\noindent\emph{Case 2:} $\vec{B}^\star \notin \mathcal{A}$. We prove that this case will never happen by contradiction. In other words, the optimal arm that represents the optimal budget solution for \CBB{} will also be budded out by the \TreeArm{} in \ZoomingBranch{} algorithm, as the time horizon is sufficiently large.

Let $\vec B^{2nd} := \argmax_{a\in \mathcal A} \mu(a)$ be the arm with the highest mean in the candidate arms. Since the number of pulls for all other suboptimal arms satisfies~(\ref{eq:expected_pulls_UCB}), the number of pulls for $\vec B^{2nd}$ grows linearly with $t$:
$$
    \mathbb E[N_T(\vec B^{2nd})] = T - \mathcal O(\log T).
$$
Since $\vec B^\star\in \TreeArm{}$, by the end of the algorithm, the \TreeArm{}'s index~($\LP{}(\TreeArm{})$)is smaller than the UCB index of $\vec B^{2nd}$. Since $\vec B^\star \in \TreeArm{}$, we have
$$
\LP{}(\TreeArm{}) \ge \LP{}(\vec B^\star) \ge \mu(\vec B^\star).
$$
then, the event that the suboptimal arm~$\vec B^{2nd}$'s UCB index be strictly greater than the optimal arm's mean $\mu(\vec B^\star)$: for $\mu(\vec B^\star) > \hat \mu_T(\vec B^{2nd})$,
\begin{align}
    & \Prx{I_T(\vec B^{2nd}) \ge \mu(\vec B^\star)} \\
    =\ & \Prx{ \hat \mu_T(\vec B^{2nd}) + c\sqrt{\frac{\log (\mathcal O(T))}{N_T(\vec{B}^{2nd})}} \ge \mu(\vec B^\star)}\\
    \le \ & \exp\left\{-\frac{\mathcal O(T)(\mu(\vec B^{2nd})-\mu(\vec B^\star))^2}{c'}\right\}\\
\end{align}
 The probability declines exponentially. Since the probability of the event in Case~2 decays exponentially with $T$, its contribution to the overall expected regret is negligible compared to the regret in Case~1. In other words, with probability tending to one as $T\to\infty$, the optimal arm $\vec{B}^\star$ is eventually pulled and becomes a candidate arm. Therefore, the overall expected regret of the \ZoomingBranch{} algorithm is dominated by the regret incurred in Case~1, and we have
$$
R(T) = \mathcal O\Biggl( \sum_{\vec{B}\in \mathcal{A}} \frac{\log T}{\Delta(\vec{B})} \Biggr).
$$

Overall the regret of \ZoomingBranch{} is controlled by the classical UCB1 guarantee, up to a constant factor, and hence the algorithm achieves near-optimal performance. This completes the proof.

\end{proof}

\end{document}